\documentclass{article}

\usepackage{microtype}
\usepackage{graphicx}
\usepackage{subfigure}
\usepackage{booktabs} 

\usepackage{hyperref}

\usepackage[accepted]{icml2024}

\usepackage{amsmath}
\usepackage{amssymb}
\usepackage{multirow}
\usepackage{mathtools}
\usepackage{amsthm}
\usepackage[english]{babel}
\usepackage{blindtext}

\newcommand{\ours}{\textsc{MMP}}
\newcommand{\x}{\mathbf{x}}

\newcommand{\z}{\mathbf{z}}
\newcommand{\Z}{\mathbf{Z}}

\newcommand{\proto}{\mathbf{a}}
\newcommand{\histo}{\text{h.}}
\newcommand{\gene}{\text{g.}}

\newcommand{\W}{\mathbf{W}}
\newcommand{\Q}{\mathbf{Q}}
\newcommand{\K}{\mathbf{K}}
\newcommand{\V}{\mathbf{V}}

\newcommand{\Ctot}{C_{\gene} + C_{\histo}}

\newcommand{\set}{\mathbb{S}}

\usepackage[capitalize,noabbrev]{cleveref}

\theoremstyle{plain}
\newtheorem{theorem}{Theorem}[section]

\newtheorem{lemma}[theorem]{Lemma}

\theoremstyle{definition}

\theoremstyle{remark}

\usepackage[textsize=tiny]{todonotes}

\icmltitlerunning{Multimodal Prototyping for Cancer Survival Prediction}

\begin{document}

\twocolumn[
\icmltitle{Multimodal Prototyping for cancer survival prediction}




\icmlsetsymbol{equal}{*}

\begin{icmlauthorlist}
\icmlauthor{Andrew H. Song}{mgb,har}
\icmlauthor{Richard J. Chen}{mgb,har}
\icmlauthor{Guillaume Jaume}{mgb,har}
\icmlauthor{Anurag Vaidya}{mgb,har,mit}
\icmlauthor{Alexander S. Baras}{jhu}
\icmlauthor{Faisal Mahmood}{mgb,har}
\end{icmlauthorlist}

\icmlaffiliation{mgb}{Mass General Brigham, Boston, MA, USA.}
\icmlaffiliation{har}{Harvard Medical School, Boston, MA, USA.}
\icmlaffiliation{mit}{Massachusetts Institute of Technology, Cambridge, MA, USA.}
\icmlaffiliation{jhu}{Johns Hopkins University School of Medicine, Baltimore, MD, USA.}

\icmlcorrespondingauthor{Andrew H. Song}{asong@bwh.harvard.edu}
\icmlcorrespondingauthor{Faisal Mahmood}{faisalmahmood@bwh.harvard.edu}

\icmlkeywords{Machine Learning, ICML}

\vskip 0.3in
]



\printAffiliationsAndNotice{}  

\begin{abstract}

Multimodal survival methods combining gigapixel histology whole-slide images (WSIs) and transcriptomic profiles are particularly promising for patient prognostication and stratification. Current approaches involve tokenizing the WSIs into smaller patches ($>10^4$ patches) and transcriptomics into gene groups, which are then integrated using a Transformer for predicting outcomes. However, this process generates many tokens, which leads to high memory requirements for computing attention and complicates post-hoc interpretability analyses. Instead, we hypothesize that we can: (1) effectively summarize the morphological content of a WSI by condensing its constituting tokens using \textit{morphological} prototypes, achieving more than $300\times$ compression; and (2) accurately characterize cellular functions by encoding the transcriptomic profile with \textit{biological pathway} prototypes, all in an unsupervised fashion. The resulting multimodal tokens are then processed by a fusion network, either with a Transformer or an optimal transport cross-alignment, which now operates with a small and fixed number of tokens without approximations. Extensive evaluation on six cancer types shows that our framework outperforms state-of-the-art methods with much less computation while unlocking new interpretability analyses. The code is available at \url{https://github.com/mahmoodlab/MMP}.

\end{abstract} 

\section{Introduction}
\label{introduction}

Patient prognostication -- the task of predicting the progression of a disease -- is a cornerstone of clinical research and can help identify novel biomarkers indicative of disease progression~\cite{song2023artificial, song2024analysis}. Prognostication is often cast as predicting survival based on a series of assays describing the patient's medical state. Due to the complexity and diverse aspects of prognostication, multimodal approaches that combine histology and omics data, such as transcriptomics~\cite{acosta2022multimodal}, are particularly promising. Histology is represented through whole-slide images (WSIs), which offer a detailed spatial depiction of the tissue, such as a tumor, with resolutions that can exceed $10^5 \times 10^5$ pixels. Differently, transcriptomics is often delineated through bulk RNA sequencing, which provides insights into gene expression. The complementary information in both modalities was shown to be predictive of survival and can be used to inform disease progression~\cite{chen2022pan, lipkova2022artificial, steyaert2023multimodal}. However, the distinct characteristics of each modality present challenges in effectively integrating them together.

WSI modeling is typically done with multiple instance learning (MIL)~\cite{ilse2018attention, campanella2019clinical, lu2021data}. This method involves (1) extracting the set of patches that constitute the WSI ($>10^4$ per WSI), (2) feeding them through a pre-trained patch encoder to generate patch embeddings, and (3) aggregating the patch embeddings with a pooling network. In contrast, transcriptomics modeling can be done using a feed-forward neural network, treating each gene expression as a tabular data entry, or by grouping them into coarse gene families~\cite{chen2021multimodal, zhou2023cross,Xu_2023_ICCV} or \emph{biological pathways}~\cite{elmarakeby2021biologically, jaume2024modeling}. The set of patch embeddings and gene groups can then be seen as tokens, which can be fed to a Transformer to derive a multimodal representation used for outcome prediction. 

However, fusing large numbers of tokens with a Transformer is computationally expensive, and most approaches resort to attention approximation~\cite{shao2021transmil, jaume2024modeling}, cross-attention~\cite{chen2021multimodal,zhou2023cross,Xu_2023_ICCV}, or token subsampling~\cite{wulczyn2020deep, Xu_2023_ICCV}. Even when employing alternatives to Transformers, such as optimal transport (OT) cross-modal alignment~\cite{duan2022multi,pramanick2022multimodal}, addressing a set of tokens remains challenging. The small size of multimodal cohorts, often just a few hundred samples, intensifies this issue, resulting in a Large-p (large input dimensionality), Small-n (small sample size) problem. Moreover, interpreting how thousands of tokens interact and contribute to patient-level prediction, which is crucial for clinical insights, presents a significant challenge. 

Instead, we hypothesize that we can summarize the patch embeddings using \emph{morphological} prototypes. Indeed, due to inherent morphological redundancy in human tissue, the histology patches that constitute the WSI can be assumed as variations of key \emph{morphologies}, \textit{e.g.,} clear cell tumor, necrosis, benign stroma, etc, which we can extract and encode. In molecular pathology, decades of research have identified \emph{biological pathways} that encode specific cellular functions~\cite{liberzon2015molecular, elmarakeby2021biologically}, which we can leverage to define \textit{pathway} prototypes. This drastically reduces the number of tokens before multimodal fusion, thereby opening up possibilities for seamless integration of diverse fusion strategies, with interpretability greatly simplified. The challenge then revolves around the extraction and encoding of meaningful multimodal prototypes.

Here, we introduce a \textbf{M}ulti\textbf{M}odal \textbf{P}rototyping framework for patient prognostication ($\ours$). Inspired by prototype-based aggregation~\cite{mialon2021a, kim2022differentiable, song2024morphological}, we construct an \textit{unsupervised} and \textit{compact} WSI representation with a Gaussian mixture model, where the mixture parameters define the slide summary, each mapping to a morphological prototype (16 to 32 prototypes).
Following existing work~\cite{jaume2024modeling}, we transform transcriptomics into a set of 50 Cancer Hallmark pathway prototypes~\cite{liberzon2015molecular}. With significantly fewer tokens, we show that multimodal Transformers can be readily applied to the joint set of histology and pathway tokens without relying on approximations. In addition, we establish a connection between Optimal Transport (OT) cross-alignment, a popular alternative for cross-modal alignment, and the Transformer cross-attention, thereby unifying both under a single framework. On six cancer cohorts from The Cancer Genome Atlas (TCGA), $\ours$ outperforms nearly all uni- and multimodal baselines with a much smaller number of operations, demonstrating the predictive performance and efficiency of prototype-based approaches. Finally, the tractable number of tokens allows visualization of bi-directional interactions between the morphological and pathway prototypes, different from previous multimodal frameworks relying on uni-directional interpretation. 

To summarize, our contributions are (1) a method for summarizing slides using morphological prototypes and summarizing transcriptomic profiles using established biological pathway prototypes; (2) a unified and memory-efficient multimodal fusion framework; (3) extensive evaluation and ablation experiments on six cancer cohorts highlighting the predictive power of the $\ours$; (4) a novel multimodal patient representation that enables novel interpretability analyses. 

\begin{figure*}[t]
   \centering
   \includegraphics[width=1\linewidth]{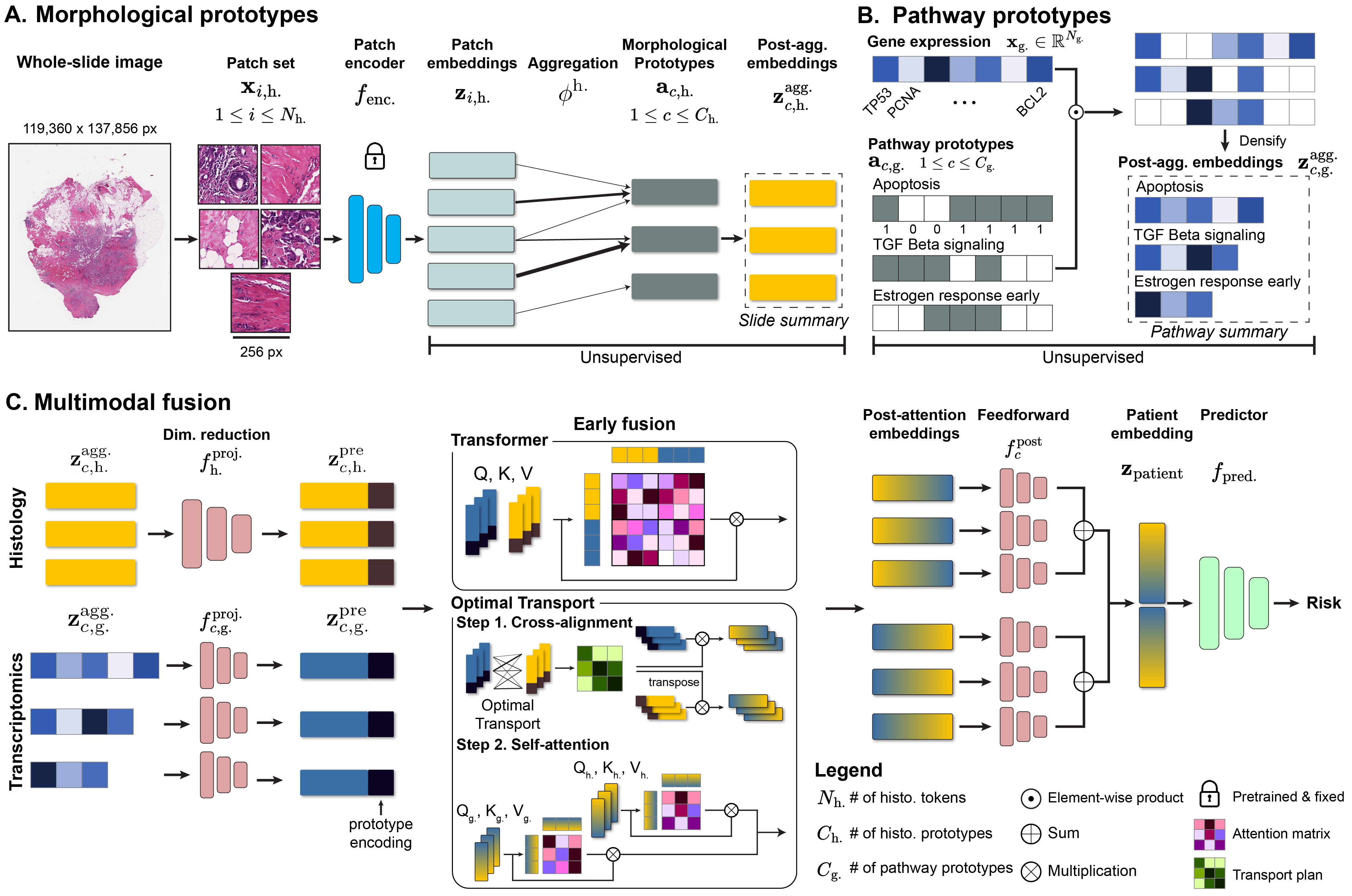}
   \caption{\textbf{Overview of $\ours$}. (A) The tessellated WSI patches (tokens) are projected to low-dimensional embeddings with a pretrained patch encoder. The patch embeddings ($N_\histo>10^4$) are aggregated to \textit{slide summary} using a small set of prototypes ($C_\histo<$32). (B) The transcriptomics data is projected onto a set of binary vectors indicating the presence of specific genes in each pathway, forming \textit{pathway summary}. (C) The post-aggregation embeddings from both modalities are first matched to the same dimension. Cross-modal interactions between histology and transcriptomics are learned with a Transformer or an Optimal Transport, with intra-modal interactions learned with Transformer-based self-attention. The attended embeddings are aggregated to form a patient-level embedding used for risk prediction.
   }
   \label{fig:main}
\end{figure*}

\section{Related Work} \label{related}

\subsection{Representing sets with prototypes}

With NLP and bioinformatics producing more datasets represented as sets, recent approaches have explored representation learning of sets with \textit{prototypes}, \emph{i.e.,} interpretable exemplars that can encode distinct concepts~\cite{snell2017prototypical, lee2019set, mialon2021a, kim2022differentiable, guo2022learning, lee2024selfsupervised}. In computational pathology, AttnMISL~\cite{yao2020whole} and H2T~\cite{VU2023handcrafted} perform K-means clustering within each WSI and use the cluster centroid embeddings as prototype (hard clustering). The proportion of patches in the cluster has also been used to represent the cluster~\cite{quiros2023mapping}. $\ours$ extends $\textsc{PANTHER}$~\cite{song2024morphological} to a multimodal setting and formalizes the prototype-based set representation with a mathematical treatise and generalizes this concept to include Gaussian mixture models, OT, and clustering.

\subsection{Prognostication with multimodal fusion}\label{sec:related_mm}
\textbf{Late fusion.} Early works exploring multimodal survival employed late fusion techniques based on merging unimodal representations, for instance using concatenation or Kronecker product~\cite{chen2020pathomic}.
While initial frameworks incorporated histology with small region-of-interests~\cite{mobadersany2018predicting, wang2021gpdbn}, the development of MIL has enabled slide-level prognostication studies combined with omics~\cite{chen2022pan, ding2023pathology, volinsky2024prediction}. However, late fusion methods are limited in modeling local cross-modal interactions potentially predictive of prognosis. \\
\textbf{Transformer fusion.} Transformers have facilitated progress in multimodal fusion by modeling interactions between cross-modal tokens (or ``early fusion''). To address the computational complexity of dealing with a large number of tokens, token subsampling can be performed~\cite{wulczyn2020deep, Xu_2023_ICCV}, or fusion can be simplified to cross-attention~\cite{chen2021multimodal,jaume2024modeling,zhou2023cross}. Present strategies allow the use of a single Transformer for merging tokens and facilitating early-fusion across various modalities~\cite{jaegle2022perceiver, girdhar2022omnivore, wang2023image, liang2023highmodality, wang2024histogenomic}. 
However, these works remain limited by the large number of histology tokens ($>10^4$). In contrast, $\ours$ does not require approximation due to the prototype-based formulation significantly reducing the number of tokens. A recent work~\cite{zhang2024prototypical} employs prototypes to remove intra- and inter-modality redundancy, but is specific to a time-discretized survival formulation. In contrast, $\ours$ allows flexible survival problem formulation.\\
\textbf{Optimal Transport fusion.} OT-based cross-alignment between sets of multimodal tokens~\cite{chen2020graph, cao2022otkge, duan2022multi, pramanick2023volta} has gained interest as an alternative to Transformers. In multimodal prognosis, MOTCat~\cite{xu2023multimodal} and $\ours$ replace Transformer-based cross-attention with an OT. 

\section{Methods}

We introduce $\ours$, a \textbf{M}ulti\textbf{M}odal \textbf{P}rototyping framework for survival prediction. We describe the construction of morphological and pathway prototypes (Section~\ref{sec:proto_encoding}) and the multimodal fusion mechanism (Section~\ref{sec:both}). We finish by describing survival prediction (Section~\ref{sec:loss}) and prototype-specific designs (Section~\ref{sec:proto_component}).

As for notations, $z$ represents a scalar, $\z$ a vector, and $\Z$ a matrix. For a set of vectors $\{ \z_c\}_{c=1}^C\in\mathbb{R}^d$, $\Z\in\mathbb{R}^{C\times d}$ represents the corresponding matrix with $\z_c$ as the $c^{\text{th}}$ row entry. The notation $\big[ \mathbf{X}, \mathbf{Z} \big]$ is used to indicate concatenation.

\subsection{Prototype-based encoding}\label{sec:proto_encoding}

\subsubsection{Morphological prototypes (Histology)}
\noindent\textbf{Preprocessing.} Given a WSI, we divide it into $N_{\histo}$ non-overlapping patches at 20$\times$ magnification ($0.5 \mu m$/pixel), forming a set of histology patches $\{\x_{i,\histo}\}_{i=1}^{N_\histo}$ of varying cardinality,  with typically $N_\histo > 10^4$. Each $\x_{i,\histo}$ is then mapped to a low-dimensional embedding using a pretrained patch encoder $f_{\text{enc}}(\cdot)$, such that $\z_{i,\histo}= f_{\text{enc}}(\x_{i,\histo})\in\mathbb{R}^D$.\\
\noindent\textbf{Aggregation.} The standard practice for WSI-based outcome prediction uses Multiple Instance Learning (MIL) to aggregate patch embeddings into a slide embedding, where patch embeddings $\set_{\histo}=\{\z_{i,\histo} \}_{i=1}^{N_\histo}$ are pooled using a learnable function $\phi^{\histo}(\cdot): \set_{\histo} \rightarrow \mathbb{R}^D$, and forms a post-aggregation slide embedding $\z_{\histo}^{\text{agg.}}=\phi^{\histo}(\set_\histo)$.
In the prototype-based approach, we take an alternate route by defining $C_\histo$ aggregation functions $\phi_{c}^{\histo}(\cdot):\set_{\histo}\rightarrow \mathbb{R}^{d_\histo}$, such that $\z_{c,\histo}^{\text{agg.}}=\phi_{c}^{\histo}(\set_{\histo})$, where $C_\histo$ is the number of prototypes. This produces a set $\set_{\text{slide}}=\{\z_{c,\histo}^{\text{agg.}}\}_{c=1}^{C_\histo}$, referred to as \emph{slide summary}.\\
\noindent\textbf{Prototypes.} We define \emph{prototypes}, denoted as $\{\proto_{c,\histo}\}_{c=1}^{C_\histo}$, $\proto_{c,\histo}\in\mathbb{R}^{d_\histo}$, such that each prototype exemplifies a unique morphology from the training set. Specifically, we apply K-means clustering on \emph{all} the patch embeddings from the training dataset to extract the $C_\histo$ cluster centroids. The prototypes $\{\proto_{c,\histo} \}_{c=1}^{C_{\histo}}$ are then used as parameters for the prototype-specific aggregation function $\phi_{c}^{\histo}$ (\textbf{Fig.~\ref{fig:main}A}).\\
\noindent\textbf{Slide summary.} Given the prototype $\mathbf{a}_{c,\histo}$ and patch embeddings $\set_{\histo}$, we can express $\phi_{c}^{\histo}$ as 
\begin{equation}\label{eq:histo_agg}
    \z_{c,\histo}^{\text{agg.}} = \phi_{c}^{\histo}(\set_{\histo}, \proto_{c,\histo}) = \sum_{i=1}^{N_\histo} g(\z_{i,\histo}, \proto_{c,\histo}),\,\,\forall c.
\end{equation}
This allows the large variable-length set of patch embeddings to be represented using a small fixed-length set comprised of prototypical tokens. We use $C_\histo \leq 32$ with typically $N_\histo > 10^4$, achieving more than 300$\times$ reduction. Eq.~\ref{eq:histo_agg} implies that aggregation is performed by summing the contribution from all embeddings in $\set_{\histo}$, the exact manner of which is determined by the mapping $g$. 

We explore three strategies for defining $g$:
hard clustering (HC), OT, and our preferred choice, Gaussian Mixture Models (GMM). We briefly explain GMM, similar to the developments in $\textsc{PANTHER}$~\cite{song2024morphological}, and defer the detailed explanations of other strategies in \textbf{Appendix~\ref{sec:histo_agg}}.\\
\noindent\textbf{GMM-based slide summarization.} With GMM as the generative model, the probability distribution for $\z_{i,\histo}$ is
\begin{equation}
    \begin{split}
   p(\z_{i,\histo};\theta) &= \sum_{c=1}^{C_{\histo}} p(c_i=c)\cdot p(\z_{i,\histo}| c_i=c)\\
   &= \sum_{c=1}^{C_{\histo}}\pi_c \cdot \mathcal{N}(\z_{i,\histo}; \boldsymbol{\mu}_c, \Sigma_c),\\ 
\end{split}
\end{equation}
where $c_i$ denotes the mixture identity of $\z_{i,\histo}$ and $\theta=\{\pi_c,\boldsymbol{\mu}_c, \Sigma_c\}_{c=1}^{C_\histo}$ denotes mixture probability, mean, and diagonal covariance. Intuitively, $\boldsymbol{\mu}_c$ and $\pi_c$ represent a morphological exemplar and the proportion of similar patterns in WSI, respectively. The posterior distribution $p(c_i=c|\z_{i,\histo})$ indirectly represents the distance between $\z_{i,\histo}$ and $\proto_{c,\histo}$ and consequently its contribution towards each element of the \textit{slide summary}. We obtain the maximum-likelihood estimate, $\widehat{\theta}=\arg\max_{\theta} \sum_{i=1}^{N_\histo} \log p(\z_{i,\histo};\theta)$, via expectation-maximization (EM), which can be performed as a feedforward network operation~\cite{dempster1977maximum, kim2022differentiable}. 

The estimated GMM parameters are concatenated to form a post-aggregation embedding, $\z_{c,\histo}^{ \text{agg.}}=[\widehat{\pi}_c,\widehat{\boldsymbol{\mu}}_c, \widehat{\Sigma}_c]\in\mathbb{R}^{d_\histo}$ with $d_\histo=2D+1$. During EM-based inference, the prototypes $\mathbf{a}_{c,\histo}$ are used as the initial parameters for the mixture means, $\boldsymbol{\mu}_c^{(0)}=\mathbf{a}_{c,\histo}$ with $\Sigma_c$ as the identity matrix. Owing to GMM's soft clustering nature, the distribution $p(c_i=c|\z_{i,\histo})$ is non-zero, implying that all elements of $\set_{\histo}$ contribute to $\z_{c,\histo}^{\text{agg.}}$ (Eq.~\ref{eq:histo_agg}). We note that deriving the \emph{slide summary} ($\set_{\text{slide}}$ or $\Z_{\histo}^{\text{agg.}}\in\mathbb{R}^{C_\histo\times d_\histo}$) from the patch embeddings is done in an \emph{unsupervised} manner, and drastically reduces the input size to the multimodal fusion model. 

\subsubsection{Pathway prototypes (Genomics)}\label{sec:gene}
We aim to define a similar compact prototypical representation of the transcriptomic profile. The transcriptomic profile spans $N_{\gene}$ gene expressions for each tissue, and is described as $\{x_{i,\gene}\}_{i=1}^{N_{\gene}}$, with $x_{i,\gene}\in\mathbb{R}$.\\
\noindent\textbf{Prototypes.} We tokenize gene expression into biological pathway prototypes, \emph{i.e.,} into groups of genes that interact in certain ways to implement previously described cellular processes~\cite{liberzon2015molecular, reimand2019pathway}. Unlike histology, the number $C_\gene$ (\textit{e.g.,} $C_\gene$=50 for Hallmark pathways) and composition of prototypes is fixed and can be defined using existing biological pathway databases. 

We define the prototypes as $\{\proto_{c,\gene} \}_{c=1}^{C_{\gene}}$, where the binary vector $\proto_{c,\gene}\in\{0,1\}^{N_\gene}$ with 1 and 0 indicates the presence and absence of a specific gene in the pathway $c$.\\
\noindent\textbf{Pathway summary.} Denoting $N_{c,\gene}$ as the number of genes in pathway $c$, we can construct $\z_{c, \gene}^{\text{agg.}}\in \mathbb{R}^{N_{c,\gene}}$, and \textit{pathway summary} $\set_{\text{path.}}=\{\z_{c, \gene}^{\text{agg.}}\}_{c=1}^{C_\gene}$ as, 
\begin{equation}
        \z_{c,\gene}^{\text{agg.}} = \phi^{\gene}(\x_{\gene}, \proto_{c,\gene}) = R(\x_{\gene} \odot \proto_{c,\gene})\in \mathbb{R}^{N_{c,\gene}},
\end{equation}
where $\x_g\in\mathbb{R}^{N_\gene}$ is the vector representation of $\{x_{i,g}\}_{i=1}^{N_{\gene}}$, $\odot$ denotes element-wise multiplication, and $R$ densifies the pathway representation by removing zero elements (\textbf{Fig.~\ref{fig:main}B}). In our work, $N_g \simeq 3\times 10^4$ and $N_{c,\gene} < 200$, achieving more than $20\times$ reduction.

To summarize, morphological and pathway prototypes are used to extract a \emph{slide summary} $\set_{\text{slide}}$ and a \emph{pathway summary} $\set_{\text{path.}}$. The morphological prototypes are defined in the patch embedding space with fixed-length, $\z_{c,\histo}^{\text{agg.}}\in\mathbb{R}^{d_\histo}$, and encode distinct morphological attributes. Differently, the pathway prototypes are defined in the raw data space with variable-length $\z_{c,\gene}^{\text{agg.}}\in\mathbb{R}^{N_{c,\gene}}$, and encode specific biological pathways.
We note that both approaches are \textit{unsupervised} and thus not require patient outcomes.   

\subsection{Multimodal fusion}\label{sec:both}
\subsubsection{Token dimension matching} Prior to multimodal fusion, we first match the dimensions of tokens from each modality. We use a linear projection $\z_{c,\histo}^{\text{pre}} = f_{\histo}^{\text{pre}}(\z_{c,\histo}^{\text{agg.}})\in\mathbb{R}^d$ for histology. For pathways, we use an MLP or self-normalizing neural networks (SNN)~\cite{Klambauer2017self} $f_{c, \gene}^{\text{pre}}$ per prototype to map variable-length representations to a common length, $\z_{c,\gene}^{\text{pre}}  = f_{c, \gene}^{\text{pre}}(\z_{c,\gene}^{\text{agg.}})\in\mathbb{R}^d$. The parameters of $f_{\histo}^{\text{pre}}$ and $\{f_{c, \gene}^{\text{pre}}\}_{c=1}^{C_{\gene}}$ are learned for each downstream task.

\subsubsection{Multimodal fusion} Inspired by multimodal early fusion methods, we learn dense intra- and cross-modal interactions between the histology and pathway tokens. We explore two strategies: Transformer attention and OT cross-alignment (\textbf{Fig.~\ref{fig:main}C}). \\
\textbf{Transformer attention.}
We introduce three learnable query, key, value matrices $\W_Q, \W_K, \W_V\in \mathbb{R}^{d\times d}$.  
Denoting $\mathbf{Q}=(
        \Q_{\gene}^{\text{T}} \hspace{1.2mm}
        \Q_{\histo}^{\text{T}})^{\text{T}}=\big(
        \Z_\gene^{\text{pre},\text{T}} \hspace{1.2mm}
        \Z_\histo^{\text{pre},\text{T}}\big)^{\text{T}}\W_Q\in\mathbb{R}^{(C_{\gene}+C_{\histo})\times d}$, 
and likewise for $\K$ and $\V$,
we can define the standard Transformer attention~\cite{vaswani2017attention, xu2023multimodal}
\begin{equation}\label{eq:transformer}
\begin{split}
\Z^{\text{post}}_{\gene+\histo}&=\begin{pmatrix}
    \Z^{\text{post}}_{\gene}\\
    \Z^{\text{post}}_{\histo}
\end{pmatrix} =\sigma\left(\frac{\Q\K^{\text{T}}}{\sqrt{d}}\right)\V\in\mathbb{R}^{(C_\gene+C_\histo)\times d}\\
    &=\sigma\left(\frac{1}{\sqrt{d}}\begin{pmatrix} \Q_\gene \K_\gene^{\text{T}} \hspace{1mm} \Q_\gene \K_\histo^{\text{T}} \\ 
\Q_\histo \K_\gene^{\text{T}} \hspace{1mm} \Q_\histo \K_\histo^{\text{T}} \\ 
\end{pmatrix}\right)\begin{pmatrix}
        \V_{\gene} \\
        \V_{\histo}
    \end{pmatrix},\\
\end{split}
\end{equation}
where $\sigma(\cdot)$ denotes row-wise softmax. Eq.~\ref{eq:transformer} illustrates how multimodal attention can be decomposed into the intra-modal self-attention ($\gene\rightarrow\gene$, $\histo\rightarrow\histo$) and cross-modal cross-attention ($\gene\rightarrow\histo$, $\histo\rightarrow\gene$). In $\ours$, the complexity of computing attention is simplified to $\mathcal{O}((\Ctot)^2)$, a considerable reduction from $\mathcal{O}((N_{\gene}+N_\histo)^2)$ in most multimodal fusion methods that do not use prototyping.

\textbf{Optimal Transport cross-alignment} Modeling cross-modal interactions can also be approached from the point of view of OT, where we aim to learn the transport plan $\mathbf{T}\in\mathbb{R}_{+}^{C_\gene\times C_{\histo}}$ with the minimal total cost between the empirical distributions $\hat{p}(\z_{\gene}^{\text{pre}})=\frac{1}{C_\gene}\sum_{c=1}^{C_\gene}\delta(\z_{c,\gene}^{\text{pre}})$ and $\hat{p}(\z_{\histo}^{\text{pre}})=\frac{1}{C_\histo}\sum_{c'=1}^{C_\histo}\delta(\z_{c',\histo}^{\text{pre}})$, and
where $\delta(\cdot)$ is a delta function. The pairwise cost $\mathbf{D}_{c,c'}$ between the two tokens is typically computed using a $L_2$ distance or negative dot product.
The estimate $\widehat{\mathbf{T}}$ is given as the solution to the entropic-regularized OT problem~\cite{Kolouri2017Optimal},
\begin{equation}
\begin{split}
    &\min_{\mathbf{T}}\sum_{c,c'}\mathbf{D}_{c,c'}\cdot \mathbf{T}_{c,c'} +\varepsilon \mathbf{T}_{c,c'}\log \mathbf{T}_{c,c'}\\
    &\,\,\text{s.t.} \sum_{c=1}^{C_\gene}\mathbf{T}_{c,c'}=1/C_{\histo} \quad\text{and}\quad \sum_{c'=1}^{C_{\histo}}\mathbf{T}_{c,c'}=1/C_\gene,\\
\end{split}
\end{equation}
where $\varepsilon$ is the regularization parameter. The optimal plan $\widehat{\mathbf{T}}$ can be obtained with the Sinkhorn algorithm~\cite{cuturi2013sinkhorn}, which can be differentiated~\cite{genevay18learning}. This enables joint learning of $f_{\histo}^{\text{pre}}$ and $\{f_{c, \gene}^{\text{pre}}\}_{c=1}^{C_\gene}$ along with the plan $\widehat{\mathbf{T}}$. Cross-alignment with $\widehat{\mathbf{T}}$, \textit{i.e.}, $\widehat{\mathbf{T}}\Z_{\histo}^{\text{pre}}$, performs $\histo\rightarrow\gene$ attention, while $\widehat{\mathbf{T}}^{\text{T}}$ performs $\gene\rightarrow\histo$ attention, $\widehat{\mathbf{T}}^{\text{T}}\Z_{\gene}^{\text{pre}}$. After the alignment,  we learn intra-modal interactions using the Transformer self-attention. Denoting $\Q_\gene=(\widehat{\mathbf{T}}\Z_{\histo}^{\text{pre}})\W_Q$ and $\Q_\histo=(\widehat{\mathbf{T}}^{\text{T}}\Z_{\gene}^{\text{pre}})\W_Q$, and likewise for $\K_\gene,\V_\gene, \K_\histo, \V_\histo$, we obtain
\begin{equation}
    \Z_\histo^{\text{post}}=\sigma\left(\frac{\Q_\histo\K_\histo^{\text{T}}}{\sqrt{d}}\right)\V_\histo,\,\Z_\gene^{\text{post}}=\sigma\left(\frac{\Q_\gene\K_\gene^{\text{T}}}{\sqrt{d}}\right)\V_\gene.
\end{equation}

\subsubsection{Connection between transformer and Optimal transport cross-alignment}\label{sec:theory}

The Transformer cross-attention and the OT cross-alignment exhibit similarities in the way attention and the transport plan are being modeled.
We can formalize these similarities to demonstrate the connection between the two. Specifically, we show that the Transformer cross-attention is similar to OT cross-alignment, under certain conditions.
\begin{lemma}
\label{lem:mot}
Let $\Z_\gene\in\mathbb{R}^{C_\gene\times d}$ and $\Z_\histo\in\mathbb{R}^{C_\histo\times d}$ be the matrix representation of the token sets $\{\z_{i,\gene} \}_{i=1}^{C_\gene}$ and $\{\z_{k,\histo} \}_{k=1}^{C_\histo}$. Let $\Z_{\gene}\W_Q^{\text{T}}\in\mathbb{R}^{C_\gene\times d}$ and $\Z_{\histo}\W^{\text{T}}\in\mathbb{R}^{C_\histo\times d}$ be the linear projections of both sets. Let $\widehat{\mathbf{T}}\in\mathbb{R}^{C_\gene\times C_\histo}_{+}$ be the optimal transport plan, i.e., the solution to the entropic-regularized, unbalanced optimal transport problem between the two projected sets. Then, $\widehat{\mathbf{T}}$ is equivalent to the Transformer cross-attention matrix, $\sigma(\Z_\gene\W_Q^{\text{T}}\W\Z_\histo^{\text{T}}/\sqrt{d})$, up to a multiplicative factor where $\sigma(\cdot)$ denotes row-wise softmax, $\{\W_Q\z_{i,\gene}\}_{i=1}^{C_\gene}$ are queries, and $\{\W\z_{k,\histo}\}_{k=1}^{C_\histo}$ are keys.
\end{lemma}
\begin{proof}
    The detailed derivation can be found in \textbf{Appendix~\ref{sec:proof}}.
\end{proof}

This lays the groundwork for $\ours$ to integrate both approaches within a single framework, rather than regarding them as fundamentally distinct approaches. This offers a platform for future innovations in multimodal strategies.

\subsection{Survival prediction}\label{sec:loss}
The post-attention embeddings are subject to a sequence of post-attention feedforward network $f^{\text{post}}$ with layer normalization (LN), averaging within each modality, and concatenation to form a patient-level embedding 
$\z_{\text{patient}} = \Big[\sum_{c=1}^{C_{\gene}} \operatorname{LN}(f^{\text{post}}(\z_{c,\gene}^{\text{post}})), \sum_{c=1}^{C_{\histo}} \operatorname{LN}(f^{\text{post}}(\z_{c,\histo}^{\text{post}}))\Big]$.
The resulting embedding is fed through a linear predictor $f_{\text{pred.}}$ for patient-level risk prediction.

We use the Cox proportional hazards loss~\cite{cox1972regression, katzman2018deepsurv, carmichael2022incorporating}, which requires training in batches to preserve the risk order within a patient group. Due to the large $N_\histo$, it is computationally challenging to form a batch for non-prototype approaches. \textbf{Appendix~\ref{section:loss_supp}} provides a detailed explanation of losses.

\subsection{Enhancing prototypes}\label{sec:proto_component}
Given that the identities of the prototypes remain consistent across patients -- \emph{e.g.,} prototype $c$ consistently represents the same morphological concept or pathway -- we can additionally inject this property into model design considerations. Specifically, we incorporate (1) a prototype-specific encoding and (2) a post-attention feed-forward network.

Prototype encoding, denoted as $\mathbf{e}_c$, can be connected to modality-specific encodings~\cite{jaegle2022perceiver, liang2023highmodality}. Specifically, we append the encodings to the embeddings before feeding them to the fusion network. We experiment with two approaches: 1) fixed one-hot encoding $\mathbf{e}_c \in \{0,1\}^{d_e}$ with $d_e=\Ctot$ and 2) random-initialized and learnable embedding $\mathbf{e}_c \in \mathbb{R}^{d_e}$ with $d_e=32$. The modified embeddings are then given as $\z_{c,\histo}^{\text{pre}}=[\z_{c,\histo}^{\text{pre}}, \mathbf{e}_c]\in\mathbb{R}^{d+d_e}$, and same for $\z_{c,\gene}^{\text{pre}}$. 

We also employ prototype-specific feedforward network (FFN) $f_{c}^{\text{post}}$ to the post-attention embeddings to learn additional nonlinearity per prototype. This differs from previous works that share $f^{\text{post}}$, which might limit expressivity. These components cannot be used for non-prototype frameworks, since 1) the patch identity is not preserved across patients and 2) large $N_\histo$ makes the use of $\mathbf{e}_c$ and $f_{c}^{\text{post}}$ infeasible.

\begin{table*}[!ht]
\centering
\small
\caption{\textbf{Survival prediction} Results for $\ours$ and other baselines for measuring disease-specific survival with C-Index. The clinical baseline includes age, sex, and cancer grade as reported in the TCGA cohort. We use the same histology feature encoder, UNI, a ViT-L/16 model pretrained on an internal histology dataset~\cite{chen2024towards}. All histology prototype-based methods share the same set of morphological prototypes with $C_{\histo}=16$. Standard deviation is reported over five runs. m.p. and p.p. denote morphological prototype and pathway prototype, respectively. The best and second-best performances are denoted by \textbf{bold} and \underline{underlined}, respectively.}
\begin{tabular}{ll|c|c|c|c|c|c|c|c|c}
\toprule
&\textbf{Dataset} & m.p. & p.p.  & BRCA & BLCA & LUAD & STAD & CRC & KIRC & Avg. ($\uparrow$)\\
\midrule
& Clinical &  & & 0.563{\tiny $\pm0.055$} & 0.570{\tiny $\pm0.033$} & 0.528{\tiny $\pm0.028$} & 0.592{\tiny $\pm0.044$} & 0.655{\tiny $\pm0.119$} & 0.602{\tiny $\pm0.066$} & 0.585  \\
\midrule
\parbox[t]{0mm}{\multirow{2}{*}{\rotatebox[origin=c]{90}{{\textbf{gene}}}}}
& Gene exp. &  & & 0.638{\tiny $\pm0.090$} & 0.627{\tiny $\pm0.055$} & 0.577{\tiny $\pm0.057$} & 0.562{\tiny $\pm0.083$} & 0.588{\tiny $\pm0.105$} & 0.681{\tiny $\pm0.072$} & 0.612  \\
& Pathways & & $\checkmark$ &0.615{\tiny $\pm0.054$} & 0.606{\tiny $\pm0.084$} & 0.626{\tiny $\pm0.077$} & 0.566{\tiny $\pm0.080$} & 0.590{\tiny $\pm0.104$} & 0.681{\tiny $\pm0.090$}  & 0.614\\
\midrule
\parbox[t]{0mm}{\multirow{6}{*}{\rotatebox[origin=c]{90}{{\textbf{histology}}}}} 
& ABMIL &&& 0.570{\tiny $\pm0.086$} & 0.550{\tiny $\pm0.039$} & 0.571{\tiny $\pm0.036$}  & 0.559{\tiny $\pm0.059$} & \underline{0.660}{\tiny $\pm0.096$} & 0.684{\tiny $\pm0.115$}  & 0.599 \\
& TransMIL &&& 0.601{\tiny $\pm0.110$} & 0.584{\tiny $\pm0.057$} & 0.547{\tiny $\pm0.054$} & 0.487{\tiny $\pm0.057$} & 0.555{\tiny $\pm0.059$} & 0.678{\tiny $\pm0.191$}  & 0.575 \\
& AttnMISL & $\checkmark$ && 0.599{\tiny $\pm0.117$} & 0.493{\tiny $\pm0.064$} & 0.627{\tiny $\pm0.076$} & 0.533{\tiny $\pm0.040$} & \textbf{0.728}{\tiny $\pm0.110$} & 0.648{\tiny $\pm0.102$}  & 0.605\\
& IB-MIL &&& 0.511{\tiny $\pm0.068$} & 0.524{\tiny $\pm0.051$} & 0.578{\tiny $\pm0.067$} & 0.525{\tiny $\pm0.061$} & 0.576{\tiny $\pm0.129$} & 0.702{\tiny $\pm0.081$}  & 0.569\\
& ILRA &&& 0.597{\tiny $\pm0.124$} & 0.581{\tiny $\pm0.055$} & 0.511{\tiny $\pm0.077$} & 0.550{\tiny $\pm0.094$} & 0.643{\tiny $\pm0.124$} & 0.651{\tiny $\pm0.164$}  & 0.589\\
& $\ours$ & $\checkmark$ &&0.669{\tiny $\pm0.119$} & 0.593{\tiny $\pm0.062$} & 0.600{\tiny $\pm0.039$} & 0.488
{\tiny $\pm0.093$}& 0.646{\tiny $\pm0.111$} & 0.701{\tiny $\pm0.177$}  & 0.611\\
\midrule
\parbox[t]{0mm}{\multirow{6}{*}{\rotatebox[origin=c]{90}{{\textbf{Multimodal}}}}}
& MCAT && $\checkmark$ & 0.648{\tiny $\pm0.100$} & 0.619{\tiny $\pm0.048$} & 0.615{\tiny $\pm0.072$} & 0.528{\tiny $\pm0.114$} & 0.578{\tiny $\pm0.136$} & 0.670{\tiny $\pm0.235$} & 0.610\\
& SurvPath  & & $\checkmark$ & 0.709{\tiny $\pm0.062$} & 0.619{\tiny $\pm0.052$} & 0.612{\tiny $\pm0.060$} & 0.556{\tiny $\pm0.136$} & 0.539{\tiny $\pm0.150$} &0.738{\tiny $\pm0.131$}  & 0.629 \\
& MOTCat && $\checkmark$ & 0.717{\tiny $\pm0.029$} & 0.622{\tiny $\pm0.064$} & 0.589{\tiny $\pm0.059$} & 0.561{\tiny $\pm0.075$} & 0.590{\tiny $\pm0.130$} & 0.708{\tiny $\pm0.104$}  & 0.631 \\
& CMTA && $\checkmark$  & 0.687{\tiny $\pm0.077$} & 0.605{\tiny $\pm0.076$} & 0.622{\tiny $\pm0.059$} & 0.547{\tiny $\pm0.088$} & 0.559{\tiny $\pm0.195$} & 0.720{\tiny $\pm0.124$}  & 0.623\\
& \textbf{$\ours_{\text{OT}}$} & \checkmark & $\checkmark$  & \textbf{0.753}{\tiny $\pm0.069$} & \underline{0.628}{\tiny $\pm0.064$} & \textbf{0.643}{\tiny $\pm0.013$} & \underline{0.580}{\tiny $\pm0.071$} & 0.636{\tiny $\pm0.120$} & \textbf{0.748}{\tiny $\pm0.099$} &  \textbf{0.665}\\
& \textbf{$\ours_{\text{Trans.}}$} & \checkmark & $\checkmark$  & \underline{0.738}{\tiny $\pm0.069$} & \textbf{0.635}{\tiny $\pm0.051$} & \underline{0.642}{\tiny $\pm0.037$} & \textbf{0.598}{\tiny $\pm0.051$} & 0.630{\tiny $\pm0.125$} & \underline{0.747{\tiny $\pm0.106$}}  & \textbf{0.665}\\

\bottomrule
\end{tabular}
\label{tab:survival_main}
\end{table*}

\section{Experiments}
\subsection{Datasets} We use publicly available The Cancer Genome Atlas (TCGA) to evaluate $\ours$ across six cancer types: Bladder urothelial carcinoma (BLCA) ($n=359$), Breast invasive carcinoma (BRCA) ($n=868$), Lung adenocarcinoma (LUAD) ($n=412$), Stomach adenocarcinoma (STAD) ($n=318$), Colon and Rectum adenocarcinoma (CRC) ($n=296$), and Kidney renal clear cell carcinoma (KIRC) ($n=340$). We train the models to predict risks for disease-specific survival (DSS)~\cite{liu2018integrated}. Following standard practice, we use 5-fold site-stratified cross-validation to mitigate batch effect~\cite{howard2021impact}. We evaluate $\ours$ with the concordance index (C-Index), which measures the concordance between the ordering based on patients' survival days and the predicted risks.

Log-2 transformed transcripts per million bulk RNA sequencing expression for all TCGA cohorts is accessed through UCSC Xena database \cite{goldman2020visualizing}. The $C_\gene=50$ Hallmark gene sets from the Molecular Signatures Database (MSigDB) \cite{subramanian2005gene,liberzon2015molecular} are used to select and organize genes into biological pathways. Hallmark gene sets (pathways) represent well-defined biological states in cancer. After organizing genes into Hallmark pathways, we obtained 4,241 unique genes across the 50 pathways, with a minimum and maximum pathway size of 31 and 199, respectively. More dataset details can be found in \textbf{Appendix~\ref{sec:dataset}}.

\subsection{Baselines}
\noindent \textbf{Histology.} We employ Attention-based MIL (ABMIL)~\cite{ilse2018attention}, ABMIL with information bottleneck (ABMIL-IB)~\cite{li2023task}, Transformer-based MIL (TransMIL)~\cite{shao2021transmil}, low-rank MIL (ILRA)~\cite{xiang2023exploring}, and prototype-based MIL (AttnMISL)~\cite{yao2020whole}. We also use the unimodal version of $\ours$.\\
\noindent \textbf{Transcriptomics.} We employ a feed-forward neural network (2-layer MLP) (non-prototype) and a baseline with pathway-specific SNNs~\cite{jaume2024modeling, zhang2024prototypical}, followed by concatenation.\\ 
\noindent \textbf{Multimodal.} We use MCAT~\cite{chen2021multimodal}, SurvPath~\cite{jaume2024modeling}, MOTCat~\cite{Xu_2023_ICCV}, and CMTA~\cite{zhou2023cross}, which all use multimodal tokenization to derive histology and omics tokens (pathways in our evaluation), followed by co-attention Transformer.\\ 
\noindent \textbf{$\ours$ variants.} We test $\ours$ with a Transformer cross-attention ($\ours_{\text{Trans.}}$) and OT cross-alignment ($\ours_{\text{OT}}$). The rest of the model comprises GMM histology aggregation with $C_\histo=16$, learnable random prototype encoding, and prototype-specific feedforward networks.

For the patch encoder, we use UNI~\cite{chen2024towards}, a DINOv2-based ViT-Large~\cite{dosovitskiy2021image, oquab2023dinov2} pretrained on $1\times 10^8$ patches sampled across $1\times 10^5$ WSIs from Mass General Brigham. We also ablate with CTransPath~\cite{wang2022transformer}, a Swin Transformer pretrained on $3.2\times 10^4$ WSIs from the TCGA, and ResNet50 pretrained on Imagenet~\cite{deng2009imagenet}. Further information on all baselines can be found in \textbf{Appendix~\ref{section:baselines}}.

\subsection{Implementation}

All models are trained with a $1\times 10^{-4}$ learning rate with cosine decay scheduler, AdamW optimizer, and $1\times 10^{-5}$ weight decay for 20 epochs. $\ours$ uses the Cox loss with a batch size of 64. Non-prototype baselines are trained with the NLL survival loss~\cite{zadeh2020bias} with a batch size of 1. During training, in MCAT, SurvPath, MOTCat, and CMTA, we randomly sample 4,096 patches per WSI to increase diversity and reduce memory. During inference, the whole WSI is used. All prototype baselines use $C_\histo=16$. 

\section{Results}

\subsection{Survival prediction}
The results are shown in \textbf{Table~\ref{tab:survival_main}}. Overall, $\ours$ outperforms all baselines (+5.4\% and +7.8\% avg. over the next-best multimodal and unimodal models) and ranks within top-2 for 5 out of 6 diseases. We highlight the main findings.

\textbf{Comparison with clinical baseline.} All multimodal baselines perform superior to the clinical baseline comprised of important prognostic variables -- age, sex, and cancer grade~\cite{bonnier1995age, rakha2010breast, tas2013age}. This demonstrates the clinical potential of multimodal frameworks for enhanced patient prognostication. Additional univariate clinical baselines are in \textbf{Appendix~\ref{sec:clinical}}.\\
\textbf{Unimodal vs. Multimodal.} All multimodal baselines (excluding MCAT) outperform the unimodal baselines (histology and transcriptomics). This aligns with previous multimodal literature showing that histology and transcriptomics contain complementary information to be leveraged for better prognostication. 
In addition for CRC, we observe unimodal histology baselines outperforming multimodal baselines, indicating that challenges remain in multimodal training dynamics of histology-omic models~\cite{gat2020removing,wang2020makes}.\\
\textbf{Prototypes vs. non-prototypes.} $\ours$ significantly outperforms all multimodal approaches that are based on prototyping  (+5.4\% avg. over the next-best model, MOTCat). While every multimodal baseline utilizes early fusion, MCAT and MOTCat learn only the uni-directional cross-modal interaction from transcriptomics to histology. Conversely, SurvPath omits histology-to-histology interactions in self-attention computation to reduce computational requirements. Overall, we attribute the superior performance of $\ours$ to our ability to 1) retain and encode morphological information predicted of prognosis in the morphological prototypes, 2) model both the intra- and cross-modal interactions without approximations, and 3) employ the Cox survival loss. The quality of the morphological prototypes is reaffirmed in the unimodal setting, where the prototype-based AttnMISL and the unimodal $\ours$ are the two best-performing models, outperforming all other approaches. \\
\textbf{Transformer vs. OT-based cross-attention.} We observe that the performance of $\ours_{\text{Trans.}}$ and $\ours_{\text{OT}}$ are on the same level, empirically confirming the connection between both approaches highlighted in Section~\ref{sec:theory}.

\subsection{Risk stratification}
We perform log-rank tests~\cite{bland2004logrank} between the high-risk and low-risk cohorts,  stratified at 50\% percentile of the risks predicted by $\ours$ and MOTCat, the next-best performing model. Specifically, we aggregate the predicted risks across all test folds to construct the cohort-level risk set. \textbf{Table~\ref{tab:risk}} shows the p-values for the log-rank tests on all 6 cancer types. With a statistical significance threshold of 0.05, we observe that $\ours$ can significantly stratify the high- and low-risk groups for all 6 cancer types, whereas MOTCat was significant for 3 cancer types. This demonstrates the strength of MMP for risk stratification over other baselines and reaffirms its clinical potential. 

\begin{table}[!ht]
\centering
\small
\caption{\textbf{Risk stratification.} We report log-rank p-values for high- and low-risk patient cohorts for $\ours$ and MOTCat. p-values below 0.05 are considered statistically significant.} 
\begin{tabular}{l|c|c|c}
\toprule
& BRCA & BLCA & LUAD \\
\midrule
MOTCAT & $6.16\times 10^{-5}$  & $8.60\times 10^{-5}$ & $9.65\times 10^{-1}$\\
\textbf{$\ours_{\text{Trans.}}$} & $3.08\times 10^{-5}$ & $4.50 \times 10^{-2}$ & $8.37\times 10^{-5}$\\
\midrule
& STAD & CRC & KIRC \\
\midrule
MOTCAT & $5.70\times 10^{-2}$ & $7.04\times 10^{-1}$ & $2.40\times 10^{-4}$\\
\textbf{$\ours_{\text{Trans.}}$} & $1.40 \times 10^{-2}$ & $3.60\times 10^{-2}$ & $2.59\times 10^{-8}$\\
\bottomrule
\end{tabular}
\label{tab:risk}
\end{table}

\subsection{Ablation study}
We perform extensive ablations of $\ours$ (\textbf{Table~\ref{tab:survival_ablation}}). We summarize our findings below. 
(1) \textbf{Number of morphological prototypes}: Larger number of morphological prototypes ($C_\histo$) generally yields better performance up to 16, then performance stagnates. To facilitate easier interpretation with fewer exemplars, we set $C_\histo=16$. 
(2) \textbf{Feature encoder}: UNI, a DINOv2-pretrained ViT-L encoder, substantially improves compared to CTransPath and ResNet50 pretrained on ImageNet. This underscores the importance of a powerful vision encoder trained on large histology datasets.
(3) \textbf{Histology aggregation}: Aggregation based on a GMM yields the best performance over optimal transport (OT) and hard clustering (HC). We hypothesize that this is due to GMM explicitly capturing sufficient statistics of patch embedding distribution, \textit{e.g.,} mixture probability and covariance, which other approaches cannot readily integrate.
(4) \textbf{Prototype encoding}: Adding prototype encoding, $\mathbf{e}_c$ and $f_c^{\text{post}}$, leads to better performance. This suggests the benefits of a prototype-specific measure that leverages the consistent prototype identity.
(5) \textbf{Fusion stage}: Early-fusion of tokens via cross-attention ($\ours$) outperforms late-fusion, which concatenates the self-attended embeddings averaged within each modality, without capturing cross-modal interactions.

We also perform unimodal $\ours$ ablations in \textbf{Appendix~\ref{sec:histology_ablations}}, to isolate the impact of histology-related design choices.

\begin{table}[!ht]
\centering
\small
\caption{\textbf{Ablation study.} C-Index and its change against $\ours$ as a single model component is modified, averaged across six cohorts.} 
\begin{tabular}{l|lcc|c}
\toprule
Ablation & \multicolumn{3}{c|}{Model} & Avg. \\
\midrule
\textbf{Full model} & \multicolumn{3}{c|}{$\ours$} & \textbf{0.665}\\
\midrule
Number of  & \multirow{2}{*}{$C_\histo=16$} & \multirow{2}{*}{$\Rightarrow$} & $C_\histo=8$ & 0.655 {\scriptsize ($-$1.5\%)}\\
histo. proto.& & & $C_\histo=32$ & 0.662 {\scriptsize ($-$0.5\%)}\\
\midrule
Histo.  & \multirow{2}{*}{UNI} & \multirow{2}{*}{$\Rightarrow$} & ResNet50 & 0.620 {\scriptsize ($-$6.8\%)}\\
enc. $f_{\text{enc}}$ & & & CTransPath & 0.643 {\scriptsize ($-$3.3\%)}\\
\midrule
Histo.  & \multirow{2}{*}{GMM} & \multirow{2}{*}{$\Rightarrow$} & OT & 0.658 {\scriptsize ($-$1.1\%)}\\
agg. $\phi_{c}^{\histo}$ & & & HC & 0.629 {\scriptsize ($-$5.4\%)}\\
\midrule
Proto.  & \multirow{2}{*}{random} & \multirow{2}{*}{$\Rightarrow$} & None & 0.652 {\scriptsize ($-$2.0\%)}\\
embed. $\mathbf{e}_c$ & & & One-hot & 0.660 {\scriptsize ($-$0.8\%)}\\
\midrule
FFN $f_c^{\text{post}}$  & Indiv. & $\Rightarrow$  & Shared & 0.658 {\scriptsize ($-$1.1\%)}\\
\midrule
Co-attention  & Trans. & $\Rightarrow$ & OT & \textbf{0.665} {\scriptsize ($-$0.0\%)}\\
\midrule
Fusion  & Early & $\Rightarrow$ & Late & 0.646 {\scriptsize ($-$2.9\%)}\\
\bottomrule
\end{tabular}
\label{tab:survival_ablation}
\end{table}

\subsection{Loss function}

We evaluate the performance of $\ours$ trained using either the Cox or the NLL loss (\textbf{Table~\ref{tab:survival_loss}} in \textbf{Appendix~\ref{sec:surv_ablation}}). The NLL loss has widely been used for multimodal frameworks as it accommodates training with a batch of a single patient, a necessity for managing many tokens. Conversely, the Cox loss requires a batch size greater than a single patient, which involves ordering patients within the batch. Applying both losses is feasible due to the reduced computational requirements in $\ours$. We observe that the Cox loss surpasses NLL loss overall (average C-Index of 0.665 \emph{vs}. 0.644). Furthermore, increasing the batch size (bs) with NLL leads to enhanced performance (0.621 with bs=1 \emph{vs.} 0.644 with bs=16), emphasizing the benefit of the reduced tokens.

\subsection{Computational complexity}
To assess the computational benefits of $\ours$, we measure the number of floating-point operations (FLOPs) for cross-attention baselines (\textbf{Table~\ref{tab:complexity}}). $\ours$ achieves at least 5$\times$ fewer giga-FLOPS, demonstrating the superior efficiency of prototyping. We observe that the aggregation  ($\ours_{\text{agg.}}$), which maps $N_\histo$ tokens to $C_\histo$ prototypes, constitutes most of the $\ours$ operations, with the fusion ($\ours_{\text{fusion}}$) requiring significantly less due to condensed token set.

\begin{table}[!ht]
\centering
\small
\caption{\textbf{Computational complexity.} Average number of tokens per WSI and average number of giga-FLOPs per patient.}
\begin{tabular}{l|c|c|c|c}
\toprule
&\multicolumn{2}{|c|}{LUAD} & \multicolumn{2}{c}{KIRC} \\
\cline{2-5}
& tokens & GFLOPs $(\downarrow)$ & tokens & GFLOPs $(\downarrow)$\\
\midrule
MCAT & 4,714 & 2.10 & 12,802 & 5.49 \\
SurvPath & 4,714 & 2.00 & 12,802 & 5.41 \\
CMTA & 4,714 & 17.2 & 12,802 & 40.1\\
\midrule
$\ours_{\text{agg.}}$ & 4,714 & 0.309 & 12,802 & 0.839\\
$\ours_{\text{fusion}}$ & 16 & 0.025 & 16 & 0.025\\
\textbf{$\ours_{\text{total}}$} & $\cdot$ & \textbf{0.334} & $\cdot$ & \textbf{0.864}\\
\bottomrule
\end{tabular}
\label{tab:complexity}
\end{table}

\begin{figure*}[t]
   \centering
   \includegraphics[width=1\linewidth]{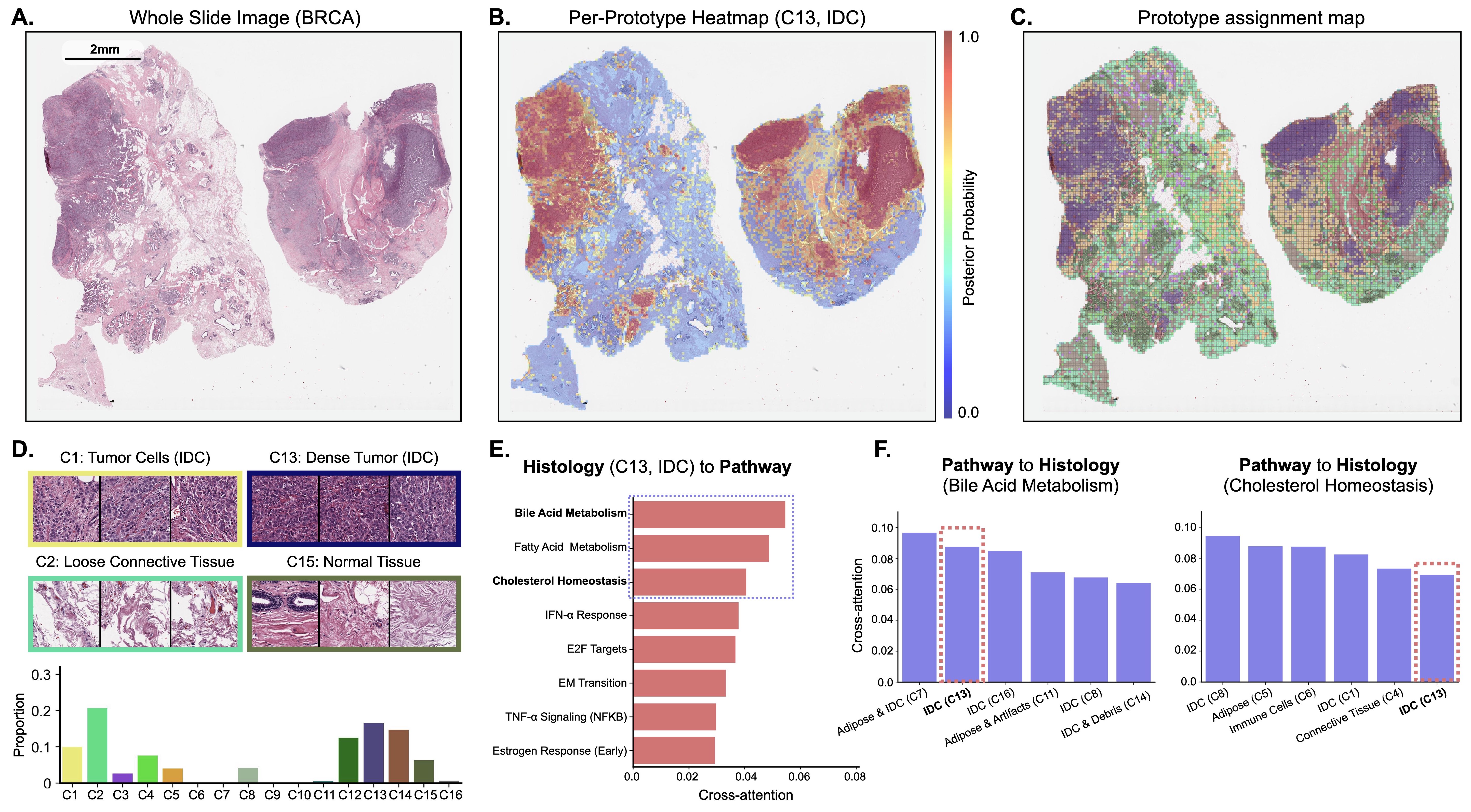}
   \vspace{-5mm}
   \caption{\textbf{Cross-modal interaction visualization}. (A) A WSI for a BRCA patient. (B) The morphological prototype heatmap for $c=13$ (C13), representing \textit{invasive ductal carcinoma} (IDC), based on the posterior distribution for C13. (C) Prototype assignment map showing the closest morphological prototype for each patch in the WSI. (D) Top-3 patches for each morphological prototype and proportion of each prototype in the WSI. (E) The top-10 pathways attending to C13. (F) Top-6 morphological prototypes attending to the pathways in (E).}
   \label{fig:heatmap}
\end{figure*}

\section{Interpretability}
\textbf{Unimodal}: As WSIs are represented with a compact set of 16 prototypes in $\ours$, we can directly visualize a prototype heatmap for prototype $c$ that corresponds to the most similar patches, by relying on the posterior $p(c_i=c|\z_{i,\histo})$ (\textbf{Fig.~\ref{fig:heatmap}A, B}, with additional examples in \textbf{Appendix~\ref{sec:heatmaps}}). The prototype assignment map can clearly show how all the prototypes are distributed in a given WSI (\textbf{Fig.~\ref{fig:heatmap}C}). For a  prototype $c$, we can also visualize the most representative patches, by querying the patch embeddings closest to $\widehat{\boldsymbol{\mu}}_c$, and its proportion in the WSI, with $\widehat{\pi}_c$ (\textbf{Fig.~\ref{fig:heatmap}D}). 

\textbf{Multimodal}: With a tractable number of histology tokens, we can visualize the cross-modal attention interactions based on cross-attention scores, from histology to pathways ($\histo\rightarrow\gene$) and pathways to histology ($\gene\rightarrow\histo$). In contrast to MCAT and MOTCat, which only model and visualize $\gene\rightarrow\histo$ interactions, \textit{i.e.,} which patches correspond to the queried pathway (histology importance), $\ours$ can also visualize $\histo\rightarrow\gene$ interactions, \textit{i.e.,} which pathways correspond to the queried prototype (pathway importance). While SurvPath also models $\histo\rightarrow\gene$ via cross-attention, since the histology patch tokens are redundant and not prototypical, visualizing $\histo\rightarrow\gene$ is intractable.

As an example, for the prototype $c=$13 (C13), which represents the dominant \textit{invasive ductal carcinoma} morphology in the BRCA WSI, we can visualize its highly-attended pathways ($\histo\rightarrow\gene$) -- bile acid metabolism, fatty acid metabolism, and cholesterol homeostasis, being important oncogenic pathways in BRCA (\textbf{Fig.~\ref{fig:heatmap}E}). This agrees with the literature that highlights the association between these pathways and breast cancer prognosis~\cite{nelson2014cholesterol, koundouros2020reprogramming, revzen2022role}. We can also visualize the highly-attended morphological prototypes for these pathways ($\gene\rightarrow\histo$), with C13 highly attended by bile acid metabolism (\textbf{Fig.~\ref{fig:heatmap}F}). Other IDC variations, such as C1 and C8, are also highly attended by these pathways. By virtue of bi-directional visualization capability, $\ours$ can elucidate tightly-linked relationships, characterized by strong bi-directional cross-attention values (C13 and bile acid metabolism), which is a unique capability over other methods that have only visualized (g. $\rightarrow$ h.) (\textbf{Fig.~\ref{fig:heatmap}F}) but not (h. $\rightarrow$ g.) (\textbf{Fig.~\ref{fig:heatmap}E}). Further discussion and visualizations are available in the \textbf{Appendix~\ref{sec:heatmaps}}.

\section{Conclusion \& Future works}
We introduced $\ours$, a prototype-based multimodal fusion framework for survival prediction in computational pathology. This framework introduces a prototype-based tokenization method that effectively reduces the number of tokens and the associated computational complexity common in multimodal fusion frameworks. Such reduction leads to improved overall prognostic performance and allows a bi-directional concept-based interpretation of how morphology and transcriptomes interact. 

We consider this an essential step forward for future multimodal prognosis research, which we believe can be extended and validated in different ways (also further detailed in \textbf{Appendix \ref{sec:app_limitations}}). First, the number of prototypes can be determined in a data-driven manner, \textit{e.g.,} using frameworks in Dirichlet processes~\cite{Lee2020A, li2022deep}. Next, instead of relying on shallow MLP for transcriptomics modeling, we can leverage the latest advances in single-cell foundation models~\cite{rosen2023universal, theodoris2023transfer, cui2024scgpt}. Finally, a validation with different outcomes, such as progression-free interval and recurrence risk~\cite{liu2018integrated}, as well as application to rare diseases for which not overfitting to a small cohort is paramount, will bring $\ours$ closer to clinical translation.

\clearpage
\section*{Acknowledgements}
The authors would like to thank Minyoung Kim at Samsung AI Center for the helpful advice on prototype-based aggregation; Ming Y. Lu and Tong Ding for setting up supervised MIL benchmarks. The authors acknowledge funding support from the Brigham and Women’s Hospital (BWH) President’s Fund, Mass General Hospital (MGH) Pathology and by
the National Institute of Health (NIH) National Institute of General Medical Sciences (NIGMS) through R35GM138216. 

\section*{Impact Statement}

This manuscript details efforts to enhance cancer prognosis through the integration of whole-slide imaging and gene expression profiling. Our study utilizes data from The Cancer Genome Atlas public database, which has been ethically and institutionally approved by all contributing sites. 
Although our research may yield societal impacts, it is important to emphasize that this study is designed solely for research applications and not yet intended for clinical use. Further larger external cohort validation will be required for realizing the clinical potential of our work.


\nocite{langley00}
\bibliography{main}

\begin{thebibliography}{102}
\providecommand{\natexlab}[1]{#1}
\providecommand{\url}[1]{\texttt{#1}}
\expandafter\ifx\csname urlstyle\endcsname\relax
  \providecommand{\doi}[1]{doi: #1}\else
  \providecommand{\doi}{doi: \begingroup \urlstyle{rm}\Url}\fi

\bibitem[Acosta et~al.(2022)Acosta, Falcone, Rajpurkar, and Topol]{acosta2022multimodal}
Acosta, J.~N., Falcone, G.~J., Rajpurkar, P., and Topol, E.~J.
\newblock {Multimodal biomedical AI}.
\newblock \emph{Nature Medicine}, 28\penalty0 (9):\penalty0 1773--1784, 2022.

\bibitem[Benamou(2003)]{benamou2003numerical}
Benamou, J.-D.
\newblock Numerical resolution of an “unbalanced” mass transport problem.
\newblock \emph{ESAIM: Mathematical Modelling and Numerical Analysis}, 37\penalty0 (5):\penalty0 851--868, 2003.

\bibitem[Bland \& Altman(2004)Bland and Altman]{bland2004logrank}
Bland, J.~M. and Altman, D.~G.
\newblock The logrank test.
\newblock \emph{Bmj}, 328\penalty0 (7447):\penalty0 1073, 2004.

\bibitem[Bonnier et~al.(1995)Bonnier, Romain, Charpin, Lejeune, Tubiana, Martin, and Piana]{bonnier1995age}
Bonnier, P., Romain, S., Charpin, C., Lejeune, C., Tubiana, N., Martin, P.-M., and Piana, L.
\newblock Age as a prognostic factor in breast cancer: relationship to pathologic and biologic features.
\newblock \emph{International journal of cancer}, 62\penalty0 (2):\penalty0 138--144, 1995.

\bibitem[Campanella et~al.(2019)Campanella, Hanna, Geneslaw, Miraflor, Werneck Krauss~Silva, Busam, Brogi, Reuter, Klimstra, and Fuchs]{campanella2019clinical}
Campanella, G., Hanna, M.~G., Geneslaw, L., Miraflor, A., Werneck Krauss~Silva, V., Busam, K.~J., Brogi, E., Reuter, V.~E., Klimstra, D.~S., and Fuchs, T.~J.
\newblock Clinical-grade computational pathology using weakly supervised deep learning on whole slide images.
\newblock \emph{Nature medicine}, 25\penalty0 (8):\penalty0 1301--1309, 2019.

\bibitem[Cao et~al.(2022)Cao, Xu, Yang, He, Cao, and Huang]{cao2022otkge}
Cao, Z., Xu, Q., Yang, Z., He, Y., Cao, X., and Huang, Q.
\newblock Otkge: Multi-modal knowledge graph embeddings via optimal transport.
\newblock \emph{Advances in Neural Information Processing Systems}, 35:\penalty0 39090--39102, 2022.

\bibitem[Carmichael et~al.(2022)Carmichael, Song, Chen, Williamson, Chen, and Mahmood]{carmichael2022incorporating}
Carmichael, I., Song, A.~H., Chen, R.~J., Williamson, D.~F., Chen, T.~Y., and Mahmood, F.
\newblock Incorporating intratumoral heterogeneity into weakly-supervised deep learning models via variance pooling.
\newblock In \emph{International Conference on Medical Image Computing and Computer-Assisted Intervention}, pp.\  387--397. Springer, 2022.

\bibitem[Chen et~al.(2020{\natexlab{a}})Chen, Gan, Cheng, Li, Carin, and Liu]{chen2020graph}
Chen, L., Gan, Z., Cheng, Y., Li, L., Carin, L., and Liu, J.
\newblock Graph optimal transport for cross-domain alignment.
\newblock In \emph{International Conference on Machine Learning}, pp.\  1542--1553. PMLR, 2020{\natexlab{a}}.

\bibitem[Chen et~al.(2020{\natexlab{b}})Chen, Lu, Wang, Williamson, Rodig, Lindeman, and Mahmood]{chen2020pathomic}
Chen, R.~J., Lu, M.~Y., Wang, J., Williamson, D.~F., Rodig, S.~J., Lindeman, N.~I., and Mahmood, F.
\newblock Pathomic fusion: an integrated framework for fusing histopathology and genomic features for cancer diagnosis and prognosis.
\newblock \emph{IEEE Transactions on Medical Imaging}, 41\penalty0 (4):\penalty0 757--770, 2020{\natexlab{b}}.

\bibitem[Chen et~al.(2021)Chen, Lu, Weng, Chen, Williamson, Manz, Shady, and Mahmood]{chen2021multimodal}
Chen, R.~J., Lu, M.~Y., Weng, W.-H., Chen, T.~Y., Williamson, D.~F., Manz, T., Shady, M., and Mahmood, F.
\newblock Multimodal co-attention transformer for survival prediction in gigapixel whole slide images.
\newblock In \emph{Proceedings of the IEEE/CVF International Conference on Computer Vision}, pp.\  4015--4025, 2021.

\bibitem[Chen et~al.(2022)Chen, Lu, Williamson, Chen, Lipkova, Noor, Shaban, Shady, Williams, Joo, et~al.]{chen2022pan}
Chen, R.~J., Lu, M.~Y., Williamson, D.~F., Chen, T.~Y., Lipkova, J., Noor, Z., Shaban, M., Shady, M., Williams, M., Joo, B., et~al.
\newblock Pan-cancer integrative histology-genomic analysis via multimodal deep learning.
\newblock \emph{Cancer Cell}, 40\penalty0 (8):\penalty0 865--878, 2022.

\bibitem[Chen et~al.(2024)Chen, Ding, Lu, Williamson, Jaume, Song, Chen, Zhang, Shao, Shaban, Williams, Oldenburg, Weishaupt, Wang, Vaidya, Le, Gerber, Sahai, Williams, and Mahmood]{chen2024towards}
Chen, R.~J., Ding, T., Lu, M.~Y., Williamson, D. F.~K., Jaume, G., Song, A.~H., Chen, B., Zhang, A., Shao, D., Shaban, M., Williams, M., Oldenburg, L., Weishaupt, L.~L., Wang, J.~J., Vaidya, A., Le, L.~P., Gerber, G., Sahai, S., Williams, W., and Mahmood, F.
\newblock Towards a general-purpose foundation model for computational pathology.
\newblock \emph{Nature Medicine}, 2024.

\bibitem[Chizat et~al.(2018)Chizat, Peyr{\'e}, Schmitzer, and Vialard]{chizat2018scaling}
Chizat, L., Peyr{\'e}, G., Schmitzer, B., and Vialard, F.-X.
\newblock Scaling algorithms for unbalanced optimal transport problems.
\newblock \emph{Mathematics of Computation}, 87\penalty0 (314):\penalty0 2563--2609, 2018.

\bibitem[Cox(1972)]{cox1972regression}
Cox, D.~R.
\newblock Regression models and life-tables.
\newblock \emph{Journal of the Royal Statistical Society: Series B (Methodological)}, 34\penalty0 (2):\penalty0 187--202, 1972.

\bibitem[Cui et~al.(2024)Cui, Wang, Maan, Pang, Luo, Duan, and Wang]{cui2024scgpt}
Cui, H., Wang, C., Maan, H., Pang, K., Luo, F., Duan, N., and Wang, B.
\newblock {scGPT: toward building a foundation model for single-cell multi-omics using generative AI}.
\newblock \emph{Nature Methods}, pp.\  1--11, 2024.

\bibitem[Cuturi(2013)]{cuturi2013sinkhorn}
Cuturi, M.
\newblock Sinkhorn distances: Lightspeed computation of optimal transport.
\newblock \emph{Advances in neural information processing systems}, 26, 2013.

\bibitem[dan Guo et~al.(2022)dan Guo, Tian, Zhang, Zhou, and Zha]{guo2022learning}
dan Guo, D., Tian, L., Zhang, M., Zhou, M., and Zha, H.
\newblock Learning prototype-oriented set representations for meta-learning.
\newblock In \emph{International Conference on Learning Representations}, 2022.

\bibitem[Dempster et~al.(1977)Dempster, Laird, and Rubin]{dempster1977maximum}
Dempster, A.~P., Laird, N.~M., and Rubin, D.~B.
\newblock Maximum likelihood from incomplete data via the em algorithm.
\newblock \emph{Journal of the royal statistical society: series B (methodological)}, 39\penalty0 (1):\penalty0 1--22, 1977.

\bibitem[Deng et~al.(2009)Deng, Dong, Socher, Li, Li, and Fei-Fei]{deng2009imagenet}
Deng, J., Dong, W., Socher, R., Li, L.-J., Li, K., and Fei-Fei, L.
\newblock Imagenet: A large-scale hierarchical image database.
\newblock In \emph{2009 IEEE conference on computer vision and pattern recognition}, pp.\  248--255. {IEEE}, 2009.

\bibitem[Ding et~al.(2023)Ding, Zhou, Metaxas, and Zhang]{ding2023pathology}
Ding, K., Zhou, M., Metaxas, D.~N., and Zhang, S.
\newblock Pathology-and-genomics multimodal transformer for survival outcome prediction.
\newblock In \emph{International Conference on Medical Image Computing and Computer-Assisted Intervention}, pp.\  622--631. Springer, 2023.

\bibitem[Dongre \& Weinberg(2019)Dongre and Weinberg]{dongre2019new}
Dongre, A. and Weinberg, R.~A.
\newblock New insights into the mechanisms of epithelial--mesenchymal transition and implications for cancer.
\newblock \emph{Nature reviews Molecular cell biology}, 20\penalty0 (2):\penalty0 69--84, 2019.

\bibitem[Dosovitskiy et~al.(2021)Dosovitskiy, Beyer, Kolesnikov, Weissenborn, Zhai, Unterthiner, Dehghani, Minderer, Heigold, Gelly, Uszkoreit, and Houlsby]{dosovitskiy2021image}
Dosovitskiy, A., Beyer, L., Kolesnikov, A., Weissenborn, D., Zhai, X., Unterthiner, T., Dehghani, M., Minderer, M., Heigold, G., Gelly, S., Uszkoreit, J., and Houlsby, N.
\newblock An image is worth 16x16 words: Transformers for image recognition at scale.
\newblock In \emph{International Conference on Learning Representations}, 2021.

\bibitem[Duan et~al.(2022)Duan, Chen, Tran, Yang, Xu, Zeng, and Chilimbi]{duan2022multi}
Duan, J., Chen, L., Tran, S., Yang, J., Xu, Y., Zeng, B., and Chilimbi, T.
\newblock Multi-modal alignment using representation codebook.
\newblock In \emph{Proceedings of the IEEE/CVF Conference on Computer Vision and Pattern Recognition}, pp.\  15651--15660, 2022.

\bibitem[Elmarakeby et~al.(2021)Elmarakeby, Hwang, Arafeh, Crowdis, Gang, Liu, AlDubayan, Salari, Kregel, Richter, et~al.]{elmarakeby2021biologically}
Elmarakeby, H.~A., Hwang, J., Arafeh, R., Crowdis, J., Gang, S., Liu, D., AlDubayan, S.~H., Salari, K., Kregel, S., Richter, C., et~al.
\newblock Biologically informed deep neural network for prostate cancer discovery.
\newblock \emph{Nature}, 598\penalty0 (7880):\penalty0 348--352, 2021.

\bibitem[Gat et~al.(2020)Gat, Schwartz, Schwing, and Hazan]{gat2020removing}
Gat, I., Schwartz, I., Schwing, A., and Hazan, T.
\newblock Removing bias in multi-modal classifiers: Regularization by maximizing functional entropies.
\newblock \emph{Advances in Neural Information Processing Systems}, 33:\penalty0 3197--3208, 2020.

\bibitem[Genevay et~al.(2018)Genevay, Peyre, and Cuturi]{genevay18learning}
Genevay, A., Peyre, G., and Cuturi, M.
\newblock Learning generative models with sinkhorn divergences.
\newblock In \emph{Proceedings of the Twenty-First International Conference on Artificial Intelligence and Statistics}, volume~84 of \emph{Proceedings of Machine Learning Research}, pp.\  1608--1617. PMLR, 09--11 Apr 2018.

\bibitem[Girdhar et~al.(2022)Girdhar, Singh, Ravi, van~der Maaten, Joulin, and Misra]{girdhar2022omnivore}
Girdhar, R., Singh, M., Ravi, N., van~der Maaten, L., Joulin, A., and Misra, I.
\newblock Omnivore: A single model for many visual modalities.
\newblock In \emph{Proceedings of the IEEE/CVF Conference on Computer Vision and Pattern Recognition}, pp.\  16102--16112, 2022.

\bibitem[Giudetti et~al.(2019)Giudetti, De~Domenico, Ragusa, Lunetti, Gaballo, Franck, Simeone, Nicolardi, De~Nuccio, Santino, et~al.]{giudetti2019specific}
Giudetti, A.~M., De~Domenico, S., Ragusa, A., Lunetti, P., Gaballo, A., Franck, J., Simeone, P., Nicolardi, G., De~Nuccio, F., Santino, A., et~al.
\newblock A specific lipid metabolic profile is associated with the epithelial mesenchymal transition program.
\newblock \emph{Biochimica et Biophysica Acta (BBA)-Molecular and Cell Biology of Lipids}, 1864\penalty0 (3):\penalty0 344--357, 2019.

\bibitem[Goldman et~al.(2020)Goldman, Craft, Hastie, Repe{\v{c}}ka, McDade, Kamath, Banerjee, Luo, Rogers, Brooks, et~al.]{goldman2020visualizing}
Goldman, M.~J., Craft, B., Hastie, M., Repe{\v{c}}ka, K., McDade, F., Kamath, A., Banerjee, A., Luo, Y., Rogers, D., Brooks, A.~N., et~al.
\newblock Visualizing and interpreting cancer genomics data via the xena platform.
\newblock \emph{Nature biotechnology}, 38\penalty0 (6):\penalty0 675--678, 2020.

\bibitem[Harrell et~al.(1982)Harrell, Califf, Pryor, Lee, and Rosati]{harrell1982evaluating}
Harrell, F.~E., Califf, R.~M., Pryor, D.~B., Lee, K.~L., and Rosati, R.~A.
\newblock Evaluating the yield of medical tests.
\newblock \emph{Jama}, 247\penalty0 (18):\penalty0 2543--2546, 1982.

\bibitem[Howard et~al.(2021)Howard, Dolezal, Kochanny, Schulte, Chen, Heij, Huo, Nanda, Olopade, Kather, et~al.]{howard2021impact}
Howard, F.~M., Dolezal, J., Kochanny, S., Schulte, J., Chen, H., Heij, L., Huo, D., Nanda, R., Olopade, O.~I., Kather, J.~N., et~al.
\newblock The impact of site-specific digital histology signatures on deep learning model accuracy and bias.
\newblock \emph{Nature communications}, 12\penalty0 (1):\penalty0 4423, 2021.

\bibitem[Ilse et~al.(2018)Ilse, Tomczak, and Welling]{ilse2018attention}
Ilse, M., Tomczak, J., and Welling, M.
\newblock Attention-based deep multiple instance learning.
\newblock In \emph{International conference on machine learning}, pp.\  2127--2136. PMLR, 2018.

\bibitem[Ishay-Ronen et~al.(2019)Ishay-Ronen, Diepenbruck, Kalathur, Sugiyama, Tiede, Ivanek, Bantug, Morini, Wang, Hess, et~al.]{ishay2019gain}
Ishay-Ronen, D., Diepenbruck, M., Kalathur, R. K.~R., Sugiyama, N., Tiede, S., Ivanek, R., Bantug, G., Morini, M.~F., Wang, J., Hess, C., et~al.
\newblock Gain fat—lose metastasis: converting invasive breast cancer cells into adipocytes inhibits cancer metastasis.
\newblock \emph{Cancer cell}, 35\penalty0 (1):\penalty0 17--32, 2019.

\bibitem[Jaegle et~al.(2022)Jaegle, Borgeaud, Alayrac, Doersch, Ionescu, Ding, Koppula, Zoran, Brock, Shelhamer, Henaff, Botvinick, Zisserman, Vinyals, and Carreira]{jaegle2022perceiver}
Jaegle, A., Borgeaud, S., Alayrac, J.-B., Doersch, C., Ionescu, C., Ding, D., Koppula, S., Zoran, D., Brock, A., Shelhamer, E., Henaff, O.~J., Botvinick, M., Zisserman, A., Vinyals, O., and Carreira, J.
\newblock Perceiver {IO}: A general architecture for structured inputs \& outputs.
\newblock In \emph{International Conference on Learning Representations}, 2022.

\bibitem[Jaume et~al.(2024)Jaume, Vaidya, Chen, Williamson, Liang, and Mahmood]{jaume2024modeling}
Jaume, G., Vaidya, A., Chen, R., Williamson, D., Liang, P., and Mahmood, F.
\newblock Modeling dense multimodal interactions between biological pathways and histology for survival prediction.
\newblock \emph{Proceedings of the IEEE/CVF Conference on Computer Vision and Pattern Recognition (CVPR)}, 2024.

\bibitem[Katzman et~al.(2018)Katzman, Shaham, Cloninger, Bates, Jiang, and Kluger]{katzman2018deepsurv}
Katzman, J.~L., Shaham, U., Cloninger, A., Bates, J., Jiang, T., and Kluger, Y.
\newblock Deepsurv: personalized treatment recommender system using a cox proportional hazards deep neural network.
\newblock \emph{BMC medical research methodology}, 18\penalty0 (1):\penalty0 1--12, 2018.

\bibitem[Kim(2022)]{kim2022differentiable}
Kim, M.
\newblock Differentiable expectation-maximization for set representation learning.
\newblock In \emph{International Conference on Learning Representations}, 2022.

\bibitem[Klambauer et~al.(2017)Klambauer, Unterthiner, Mayr, and Hochreiter]{Klambauer2017self}
Klambauer, G., Unterthiner, T., Mayr, A., and Hochreiter, S.
\newblock Self-normalizing neural networks.
\newblock In \emph{Advances in Neural Information Processing Systems}, volume~30. Curran Associates, Inc., 2017.

\bibitem[Kolouri et~al.(2017)Kolouri, Park, Thorpe, Slepcev, and Rohde]{Kolouri2017Optimal}
Kolouri, S., Park, S.~R., Thorpe, M., Slepcev, D., and Rohde, G.~K.
\newblock Optimal mass transport: Signal processing and machine-learning applications.
\newblock \emph{IEEE Signal Processing Magazine}, 34\penalty0 (4):\penalty0 43--59, 2017.
\newblock \doi{10.1109/MSP.2017.2695801}.

\bibitem[Koundouros \& Poulogiannis(2020)Koundouros and Poulogiannis]{koundouros2020reprogramming}
Koundouros, N. and Poulogiannis, G.
\newblock Reprogramming of fatty acid metabolism in cancer.
\newblock \emph{British journal of cancer}, 122\penalty0 (1):\penalty0 4--22, 2020.

\bibitem[Kvamme et~al.(2019)Kvamme, Borgan, and Scheel]{kvamme2019time}
Kvamme, H., Borgan, {\O}., and Scheel, I.
\newblock Time-to-event prediction with neural networks and cox regression.
\newblock \emph{arXiv preprint arXiv:1907.00825}, 2019.

\bibitem[Lee et~al.(2024)Lee, Lee, Ko, Kawaguchi, Lee, and Hwang]{lee2024selfsupervised}
Lee, D.~B., Lee, S., Ko, J., Kawaguchi, K., Lee, J., and Hwang, S.~J.
\newblock Self-supervised dataset distillation for transfer learning.
\newblock In \emph{The Twelfth International Conference on Learning Representations}, 2024.

\bibitem[Lee et~al.(2019)Lee, Lee, Kim, Kosiorek, Choi, and Teh]{lee2019set}
Lee, J., Lee, Y., Kim, J., Kosiorek, A., Choi, S., and Teh, Y.~W.
\newblock Set transformer: A framework for attention-based permutation-invariant neural networks.
\newblock In \emph{International conference on machine learning}, pp.\  3744--3753. PMLR, 2019.

\bibitem[Lee et~al.(2020)Lee, Ha, Zhang, and Kim]{Lee2020A}
Lee, S., Ha, J., Zhang, D., and Kim, G.
\newblock A neural dirichlet process mixture model for task-free continual learning.
\newblock In \emph{International Conference on Learning Representations}, 2020.

\bibitem[Li \& Dewey(2011)Li and Dewey]{li2011rsem}
Li, B. and Dewey, C.~N.
\newblock {RSEM}: accurate transcript quantification from rna-seq data with or without a reference genome.
\newblock \emph{BMC bioinformatics}, 12:\penalty0 1--16, 2011.

\bibitem[Li et~al.(2023)Li, Zhu, Zhang, Sun, Shui, Kuang, Zheng, and Yang]{li2023task}
Li, H., Zhu, C., Zhang, Y., Sun, Y., Shui, Z., Kuang, W., Zheng, S., and Yang, L.
\newblock Task-specific fine-tuning via variational information bottleneck for weakly-supervised pathology whole slide image classification.
\newblock In \emph{Proceedings of the IEEE/CVF Conference on Computer Vision and Pattern Recognition}, pp.\  7454--7463, 2023.

\bibitem[Li et~al.(2022)Li, Li, Jiang, and Xia]{li2022deep}
Li, N., Li, W., Jiang, Y., and Xia, S.-T.
\newblock Deep dirichlet process mixture models.
\newblock In \emph{Uncertainty in Artificial Intelligence}, pp.\  1138--1147. PMLR, 2022.

\bibitem[Liang et~al.(2023)Liang, Lyu, Fan, Tsaw, Liu, Mo, Yogatama, Morency, and Salakhutdinov]{liang2023highmodality}
Liang, P.~P., Lyu, Y., Fan, X., Tsaw, J., Liu, Y., Mo, S., Yogatama, D., Morency, L.-P., and Salakhutdinov, R.
\newblock High-modality multimodal transformer: Quantifying modality \& interaction heterogeneity for high-modality representation learning.
\newblock \emph{Transactions on Machine Learning Research}, 2023.
\newblock ISSN 2835-8856.

\bibitem[Liberzon et~al.(2015)Liberzon, Birger, Thorvaldsd{\'o}ttir, Ghandi, Mesirov, and Tamayo]{liberzon2015molecular}
Liberzon, A., Birger, C., Thorvaldsd{\'o}ttir, H., Ghandi, M., Mesirov, J.~P., and Tamayo, P.
\newblock The molecular signatures database hallmark gene set collection.
\newblock \emph{Cell systems}, 1\penalty0 (6):\penalty0 417--425, 2015.

\bibitem[Lipkova et~al.(2022)Lipkova, Chen, Chen, Lu, Barbieri, Shao, Vaidya, Chen, Zhuang, Williamson, et~al.]{lipkova2022artificial}
Lipkova, J., Chen, R.~J., Chen, B., Lu, M.~Y., Barbieri, M., Shao, D., Vaidya, A.~J., Chen, C., Zhuang, L., Williamson, D.~F., et~al.
\newblock Artificial intelligence for multimodal data integration in oncology.
\newblock \emph{Cancer cell}, 40\penalty0 (10):\penalty0 1095--1110, 2022.

\bibitem[Liu et~al.(2018)Liu, Lichtenberg, Hoadley, Poisson, Lazar, Cherniack, Kovatich, Benz, Levine, Lee, et~al.]{liu2018integrated}
Liu, J., Lichtenberg, T., Hoadley, K.~A., Poisson, L.~M., Lazar, A.~J., Cherniack, A.~D., Kovatich, A.~J., Benz, C.~C., Levine, D.~A., Lee, A.~V., et~al.
\newblock {An integrated TCGA pan-cancer clinical data resource to drive high-quality survival outcome analytics}.
\newblock \emph{Cell}, 173\penalty0 (2):\penalty0 400--416, 2018.

\bibitem[Loo et~al.(2021)Loo, Toh, Xie, Pathak, Tan, Ma, Lee, Shatishwaran, Yeo, Yuan, et~al.]{loo2021fatty}
Loo, S.~Y., Toh, L.~P., Xie, W.~H., Pathak, E., Tan, W., Ma, S., Lee, M.~Y., Shatishwaran, S., Yeo, J. Z.~Z., Yuan, J., et~al.
\newblock Fatty acid oxidation is a druggable gateway regulating cellular plasticity for driving metastasis in breast cancer.
\newblock \emph{Science Advances}, 7\penalty0 (41):\penalty0 eabh2443, 2021.

\bibitem[Lu et~al.(2021)Lu, Williamson, Chen, Chen, Barbieri, and Mahmood]{lu2021data}
Lu, M.~Y., Williamson, D.~F., Chen, T.~Y., Chen, R.~J., Barbieri, M., and Mahmood, F.
\newblock Data-efficient and weakly supervised computational pathology on whole-slide images.
\newblock \emph{Nature biomedical engineering}, 5\penalty0 (6):\penalty0 555--570, 2021.

\bibitem[Mialon et~al.(2021)Mialon, Chen, d'Aspremont, and Mairal]{mialon2021a}
Mialon, G., Chen, D., d'Aspremont, A., and Mairal, J.
\newblock A trainable optimal transport embedding for feature aggregation and its relationship to attention.
\newblock In \emph{International Conference on Learning Representations}, 2021.

\bibitem[Mobadersany et~al.(2018)Mobadersany, Yousefi, Amgad, Gutman, Barnholtz-Sloan, Vel{\'a}zquez~Vega, Brat, and Cooper]{mobadersany2018predicting}
Mobadersany, P., Yousefi, S., Amgad, M., Gutman, D.~A., Barnholtz-Sloan, J.~S., Vel{\'a}zquez~Vega, J.~E., Brat, D.~J., and Cooper, L.~A.
\newblock Predicting cancer outcomes from histology and genomics using convolutional networks.
\newblock \emph{Proceedings of the National Academy of Sciences}, 115\penalty0 (13):\penalty0 E2970--E2979, 2018.

\bibitem[Nelson et~al.(2014)Nelson, Chang, and McDonnell]{nelson2014cholesterol}
Nelson, E.~R., Chang, C.-y., and McDonnell, D.~P.
\newblock Cholesterol and breast cancer pathophysiology.
\newblock \emph{Trends in Endocrinology \& Metabolism}, 25\penalty0 (12):\penalty0 649--655, 2014.

\bibitem[Olea-Flores et~al.(2018)Olea-Flores, Ju{\'a}rez-Cruz, Mendoza-Catal{\'a}n, Padilla-Benavides, and Navarro-Tito]{olea2018signaling}
Olea-Flores, M., Ju{\'a}rez-Cruz, J.~C., Mendoza-Catal{\'a}n, M.~A., Padilla-Benavides, T., and Navarro-Tito, N.
\newblock Signaling pathways induced by leptin during epithelial--mesenchymal transition in breast cancer.
\newblock \emph{International journal of molecular sciences}, 19\penalty0 (11):\penalty0 3493, 2018.

\bibitem[Olea-Flores et~al.(2020)Olea-Flores, Ju{\'a}rez-Cruz, Zu{\~n}iga-Eulogio, Acosta, Garc{\'\i}a-Rodr{\'\i}guez, Zacapala-Gomez, Mendoza-Catal{\'a}n, Ortiz-Ortiz, Ortu{\~n}o-Pineda, and Navarro-Tito]{olea2020new}
Olea-Flores, M., Ju{\'a}rez-Cruz, J.~C., Zu{\~n}iga-Eulogio, M.~D., Acosta, E., Garc{\'\i}a-Rodr{\'\i}guez, E., Zacapala-Gomez, A.~E., Mendoza-Catal{\'a}n, M.~A., Ortiz-Ortiz, J., Ortu{\~n}o-Pineda, C., and Navarro-Tito, N.
\newblock New actors driving the epithelial--mesenchymal transition in cancer: The role of leptin.
\newblock \emph{Biomolecules}, 10\penalty0 (12):\penalty0 1676, 2020.

\bibitem[Oquab et~al.(2023)Oquab, Darcet, Moutakanni, Vo, Szafraniec, Khalidov, Fernandez, Haziza, Massa, El-Nouby, et~al.]{oquab2023dinov2}
Oquab, M., Darcet, T., Moutakanni, T., Vo, H., Szafraniec, M., Khalidov, V., Fernandez, P., Haziza, D., Massa, F., El-Nouby, A., et~al.
\newblock Dinov2: Learning robust visual features without supervision.
\newblock \emph{arXiv preprint arXiv:2304.07193}, 2023.

\bibitem[P{\"o}lsterl(2020)]{sksurv}
P{\"o}lsterl, S.
\newblock scikit-survival: A library for time-to-event analysis built on top of scikit-learn.
\newblock \emph{Journal of Machine Learning Research}, 21\penalty0 (212):\penalty0 1--6, 2020.

\bibitem[Pramanick et~al.(2022)Pramanick, Roy, and Patel]{pramanick2022multimodal}
Pramanick, S., Roy, A., and Patel, V.~M.
\newblock Multimodal learning using optimal transport for sarcasm and humor detection.
\newblock In \emph{Proceedings of the IEEE/CVF Winter Conference on Applications of Computer Vision}, pp.\  3930--3940, 2022.

\bibitem[Pramanick et~al.(2023)Pramanick, Jing, Nag, Zhu, Shah, LeCun, and Chellappa]{pramanick2023volta}
Pramanick, S., Jing, L., Nag, S., Zhu, J., Shah, H.~J., LeCun, Y., and Chellappa, R.
\newblock Vo{LTA}: Vision-language transformer with weakly-supervised local-feature alignment.
\newblock \emph{Transactions on Machine Learning Research}, 2023.
\newblock ISSN 2835-8856.

\bibitem[Quiros et~al.(2023)Quiros, Coudray, Yeaton, Yang, Liu, Le, Chiriboga, Karimkhan, Narula, Moore, Park, Pass, Moreira, Quesne, Tsirigos, and Yuan]{quiros2023mapping}
Quiros, A.~C., Coudray, N., Yeaton, A., Yang, X., Liu, B., Le, H., Chiriboga, L., Karimkhan, A., Narula, N., Moore, D.~A., Park, C.~Y., Pass, H., Moreira, A.~L., Quesne, J.~L., Tsirigos, A., and Yuan, K.
\newblock Mapping the landscape of histomorphological cancer phenotypes using self-supervised learning on unlabeled, unannotated pathology slides, 2023.

\bibitem[Rakha et~al.(2010)Rakha, Reis-Filho, Baehner, Dabbs, Decker, Eusebi, Fox, Ichihara, Jacquemier, Lakhani, et~al.]{rakha2010breast}
Rakha, E.~A., Reis-Filho, J.~S., Baehner, F., Dabbs, D.~J., Decker, T., Eusebi, V., Fox, S.~B., Ichihara, S., Jacquemier, J., Lakhani, S.~R., et~al.
\newblock Breast cancer prognostic classification in the molecular era: the role of histological grade.
\newblock \emph{Breast cancer research}, 12:\penalty0 1--12, 2010.

\bibitem[Reimand et~al.(2019)Reimand, Isserlin, Voisin, Kucera, Tannus-Lopes, Rostamianfar, Wadi, Meyer, Wong, Xu, et~al.]{reimand2019pathway}
Reimand, J., Isserlin, R., Voisin, V., Kucera, M., Tannus-Lopes, C., Rostamianfar, A., Wadi, L., Meyer, M., Wong, J., Xu, C., et~al.
\newblock {Pathway enrichment analysis and visualization of omics data using g: Profiler, GSEA, Cytoscape and EnrichmentMap}.
\newblock \emph{Nature protocols}, 14\penalty0 (2):\penalty0 482--517, 2019.

\bibitem[Re{\v{z}}en et~al.(2022)Re{\v{z}}en, Rozman, Kov{\'a}cs, Kov{\'a}cs, Sipos, Bai, and Mik{\'o}]{revzen2022role}
Re{\v{z}}en, T., Rozman, D., Kov{\'a}cs, T., Kov{\'a}cs, P., Sipos, A., Bai, P., and Mik{\'o}, E.
\newblock The role of bile acids in carcinogenesis.
\newblock \emph{Cellular and molecular life sciences}, 79\penalty0 (5):\penalty0 243, 2022.

\bibitem[Rosen et~al.(2023)Rosen, Roohani, Agrawal, Samotorcan, Consortium, Quake, and Leskovec]{rosen2023universal}
Rosen, Y., Roohani, Y., Agrawal, A., Samotorcan, L., Consortium, T.~S., Quake, S.~R., and Leskovec, J.
\newblock Universal cell embeddings: A foundation model for cell biology.
\newblock \emph{bioRxiv}, pp.\  2023--11, 2023.

\bibitem[Shao et~al.(2021)Shao, Bian, Chen, Wang, Zhang, Ji, et~al.]{shao2021transmil}
Shao, Z., Bian, H., Chen, Y., Wang, Y., Zhang, J., Ji, X., et~al.
\newblock Transmil: Transformer based correlated multiple instance learning for whole slide image classification.
\newblock \emph{Advances in Neural Information Processing Systems}, 34:\penalty0 2136--2147, 2021.

\bibitem[Snell et~al.(2017)Snell, Swersky, and Zemel]{snell2017prototypical}
Snell, J., Swersky, K., and Zemel, R.
\newblock Prototypical networks for few-shot learning.
\newblock \emph{Advances in neural information processing systems}, 30, 2017.

\bibitem[Song et~al.(2023)Song, Jaume, Williamson, Lu, Vaidya, Miller, and Mahmood]{song2023artificial}
Song, A.~H., Jaume, G., Williamson, D.~F., Lu, M.~Y., Vaidya, A., Miller, T.~R., and Mahmood, F.
\newblock Artificial intelligence for digital and computational pathology.
\newblock \emph{Nature Reviews Bioengineering}, 1\penalty0 (12):\penalty0 930--949, 2023.

\bibitem[Song et~al.(2024{\natexlab{a}})Song, Chen, Ding, Williamson, Jaume, and Mahmood]{song2024morphological}
Song, A.~H., Chen, R.~J., Ding, T., Williamson, D.~F., Jaume, G., and Mahmood, F.
\newblock Morphological prototyping for unsupervised slide representation learning in computational pathology.
\newblock In \emph{Proceedings of the IEEE/CVF Conference on Computer Vision and Pattern Recognition}, 2024{\natexlab{a}}.

\bibitem[Song et~al.(2024{\natexlab{b}})Song, Williams, Williamson, Chow, Jaume, Gao, Zhang, Chen, Baras, Serafin, et~al.]{song2024analysis}
Song, A.~H., Williams, M., Williamson, D.~F., Chow, S.~S., Jaume, G., Gao, G., Zhang, A., Chen, B., Baras, A.~S., Serafin, R., et~al.
\newblock {Analysis of 3D pathology samples using weakly supervised AI}.
\newblock \emph{Cell}, 187\penalty0 (10):\penalty0 2502--2520, 2024{\natexlab{b}}.

\bibitem[Steyaert et~al.(2023)Steyaert, Pizurica, Nagaraj, Khandelwal, Hernandez-Boussard, Gentles, and Gevaert]{steyaert2023multimodal}
Steyaert, S., Pizurica, M., Nagaraj, D., Khandelwal, P., Hernandez-Boussard, T., Gentles, A.~J., and Gevaert, O.
\newblock Multimodal data fusion for cancer biomarker discovery with deep learning.
\newblock \emph{Nature Machine Intelligence}, 5\penalty0 (4):\penalty0 351--362, 2023.

\bibitem[Subramanian et~al.(2005)Subramanian, Tamayo, Mootha, Mukherjee, Ebert, Gillette, Paulovich, Pomeroy, Golub, Lander, et~al.]{subramanian2005gene}
Subramanian, A., Tamayo, P., Mootha, V.~K., Mukherjee, S., Ebert, B.~L., Gillette, M.~A., Paulovich, A., Pomeroy, S.~L., Golub, T.~R., Lander, E.~S., et~al.
\newblock Gene set enrichment analysis: a knowledge-based approach for interpreting genome-wide expression profiles.
\newblock \emph{Proceedings of the National Academy of Sciences}, 102\penalty0 (43):\penalty0 15545--15550, 2005.

\bibitem[Tas et~al.(2013)Tas, Ciftci, Kilic, and Karabulut]{tas2013age}
Tas, F., Ciftci, R., Kilic, L., and Karabulut, S.
\newblock Age is a prognostic factor affecting survival in lung cancer patients.
\newblock \emph{Oncology letters}, 6\penalty0 (5):\penalty0 1507--1513, 2013.

\bibitem[Theodoris et~al.(2023)Theodoris, Xiao, Chopra, Chaffin, Al~Sayed, Hill, Mantineo, Brydon, Zeng, Liu, et~al.]{theodoris2023transfer}
Theodoris, C.~V., Xiao, L., Chopra, A., Chaffin, M.~D., Al~Sayed, Z.~R., Hill, M.~C., Mantineo, H., Brydon, E.~M., Zeng, Z., Liu, X.~S., et~al.
\newblock Transfer learning enables predictions in network biology.
\newblock \emph{Nature}, 618\penalty0 (7965):\penalty0 616--624, 2023.

\bibitem[Uno et~al.(2011)Uno, Cai, Pencina, D'Agostino, and Wei]{uno2011c}
Uno, H., Cai, T., Pencina, M.~J., D'Agostino, R.~B., and Wei, L.-J.
\newblock On the c-statistics for evaluating overall adequacy of risk prediction procedures with censored survival data.
\newblock \emph{Statistics in medicine}, 30\penalty0 (10):\penalty0 1105--1117, 2011.

\bibitem[Vaidya et~al.(2024)Vaidya, Chen, Williamson, Song, Jaume, Yang, Hartvigsen, Dyer, Lu, Lipkova, et~al.]{vaidya2024demographic}
Vaidya, A., Chen, R.~J., Williamson, D.~F., Song, A.~H., Jaume, G., Yang, Y., Hartvigsen, T., Dyer, E.~C., Lu, M.~Y., Lipkova, J., et~al.
\newblock Demographic bias in misdiagnosis by computational pathology models.
\newblock \emph{Nature Medicine}, 30\penalty0 (4):\penalty0 1174--1190, 2024.

\bibitem[Vaswani et~al.(2017)Vaswani, Shazeer, Parmar, Uszkoreit, Jones, Gomez, Kaiser, and Polosukhin]{vaswani2017attention}
Vaswani, A., Shazeer, N., Parmar, N., Uszkoreit, J., Jones, L., Gomez, A.~N., Kaiser, {\L}., and Polosukhin, I.
\newblock Attention is all you need.
\newblock \emph{Advances in neural information processing systems}, 30, 2017.

\bibitem[Volinsky-Fremond et~al.(2024)Volinsky-Fremond, Horeweg, Andani, Barkey~Wolf, Lafarge, de~Kroon, {\O}rtoft, H{\o}gdall, Dijkstra, Jobsen, et~al.]{volinsky2024prediction}
Volinsky-Fremond, S., Horeweg, N., Andani, S., Barkey~Wolf, J., Lafarge, M.~W., de~Kroon, C.~D., {\O}rtoft, G., H{\o}gdall, E., Dijkstra, J., Jobsen, J.~J., et~al.
\newblock Prediction of recurrence risk in endometrial cancer with multimodal deep learning.
\newblock \emph{Nature Medicine}, pp.\  1--12, 2024.

\bibitem[Vu et~al.(2023)Vu, Rajpoot, Raza, and Rajpoot]{VU2023handcrafted}
Vu, Q.~D., Rajpoot, K., Raza, S. E.~A., and Rajpoot, N.
\newblock {Handcrafted Histological Transformer (H2T): Unsupervised representation of whole slide images}.
\newblock \emph{Medical Image Analysis}, 85:\penalty0 102743, 2023.
\newblock ISSN 1361-8415.

\bibitem[Wang et~al.(2020)Wang, Tran, and Feiszli]{wang2020makes}
Wang, W., Tran, D., and Feiszli, M.
\newblock What makes training multi-modal classification networks hard?
\newblock In \emph{Proceedings of the IEEE/CVF conference on computer vision and pattern recognition}, pp.\  12695--12705, 2020.

\bibitem[Wang et~al.(2023)Wang, Bao, Dong, Bjorck, Peng, Liu, Aggarwal, Mohammed, Singhal, Som, et~al.]{wang2023image}
Wang, W., Bao, H., Dong, L., Bjorck, J., Peng, Z., Liu, Q., Aggarwal, K., Mohammed, O.~K., Singhal, S., Som, S., et~al.
\newblock {Image as a Foreign Language: BEiT Pretraining for Vision and Vision-Language Tasks}.
\newblock In \emph{Proceedings of the IEEE/CVF Conference on Computer Vision and Pattern Recognition}, pp.\  19175--19186, 2023.

\bibitem[Wang et~al.(2022)Wang, Yang, Zhang, Wang, Zhang, Yang, Huang, and Han]{wang2022transformer}
Wang, X., Yang, S., Zhang, J., Wang, M., Zhang, J., Yang, W., Huang, J., and Han, X.
\newblock Transformer-based unsupervised contrastive learning for histopathological image classification.
\newblock \emph{Medical image analysis}, 81:\penalty0 102559, 2022.

\bibitem[Wang et~al.(2012)Wang, Lehu{\'e}d{\'e}, Laurent, Dirat, Dauvillier, Bochet, Le~Gonidec, Escourrou, Valet, and Muller]{wang2012adipose}
Wang, Y.-Y., Lehu{\'e}d{\'e}, C., Laurent, V., Dirat, B., Dauvillier, S., Bochet, L., Le~Gonidec, S., Escourrou, G., Valet, P., and Muller, C.
\newblock Adipose tissue and breast epithelial cells: a dangerous dynamic duo in breast cancer.
\newblock \emph{Cancer letters}, 324\penalty0 (2):\penalty0 142--151, 2012.

\bibitem[Wang et~al.(2021)Wang, Li, Wang, and Li]{wang2021gpdbn}
Wang, Z., Li, R., Wang, M., and Li, A.
\newblock {GPDBN: deep bilinear network integrating both genomic data and pathological images for breast cancer prognosis prediction}.
\newblock \emph{Bioinformatics}, 37\penalty0 (18):\penalty0 2963--2970, 2021.

\bibitem[Wang et~al.(2024)Wang, Zhang, Xu, Imoto, Chen, and Song]{wang2024histogenomic}
Wang, Z., Zhang, Y., Xu, Y., Imoto, S., Chen, H., and Song, J.
\newblock Histo-genomic knowledge distillation for cancer prognosis from histopathology whole slide images, 2024.

\bibitem[Wong(1986)]{wong1986theory}
Wong, W.~H.
\newblock Theory of partial likelihood.
\newblock \emph{The Annals of statistics}, pp.\  88--123, 1986.

\bibitem[Wu \& Zhou(2010)Wu and Zhou]{wu2010tnf}
Wu, Y.-d. and Zhou, B.
\newblock {TNF-$\alpha$/NF-$\kappa$B/Snail pathway in cancer cell migration and invasion}.
\newblock \emph{British journal of cancer}, 102\penalty0 (4):\penalty0 639--644, 2010.

\bibitem[Wulczyn et~al.(2020)Wulczyn, Steiner, Xu, Sadhwani, Wang, Flament-Auvigne, Mermel, Chen, Liu, and Stumpe]{wulczyn2020deep}
Wulczyn, E., Steiner, D.~F., Xu, Z., Sadhwani, A., Wang, H., Flament-Auvigne, I., Mermel, C.~H., Chen, P.-H.~C., Liu, Y., and Stumpe, M.~C.
\newblock Deep learning-based survival prediction for multiple cancer types using histopathology images.
\newblock \emph{PloS ONE}, 15\penalty0 (6), 2020.

\bibitem[Xiang \& Zhang(2022)Xiang and Zhang]{xiang2022exploring}
Xiang, J. and Zhang, J.
\newblock Exploring low-rank property in multiple instance learning for whole slide image classification.
\newblock In \emph{The Eleventh International Conference on Learning Representations}, 2022.

\bibitem[Xiang \& Zhang(2023)Xiang and Zhang]{xiang2023exploring}
Xiang, J. and Zhang, J.
\newblock Exploring low-rank property in multiple instance learning for whole slide image classification.
\newblock In \emph{The Eleventh International Conference on Learning Representations}, 2023.

\bibitem[Xiong et~al.(2021)Xiong, Zeng, Chakraborty, Tan, Fung, Li, and Singh]{xiong2021nystromformer}
Xiong, Y., Zeng, Z., Chakraborty, R., Tan, M., Fung, G., Li, Y., and Singh, V.
\newblock Nystr{\"o}mformer: A nystr{\"o}m-based algorithm for approximating self-attention.
\newblock In \emph{Proceedings of the AAAI Conference on Artificial Intelligence}, volume~35, pp.\  14138--14148, 2021.

\bibitem[Xu et~al.(2023)Xu, Zhu, and Clifton]{xu2023multimodal}
Xu, P., Zhu, X., and Clifton, D.~A.
\newblock Multimodal learning with transformers: A survey.
\newblock \emph{IEEE Transactions on Pattern Analysis and Machine Intelligence}, 2023.

\bibitem[Xu \& Chen(2023)Xu and Chen]{Xu_2023_ICCV}
Xu, Y. and Chen, H.
\newblock Multimodal optimal transport-based co-attention transformer with global structure consistency for survival prediction.
\newblock In \emph{Proceedings of the IEEE/CVF International Conference on Computer Vision (ICCV)}, pp.\  21241--21251, October 2023.

\bibitem[Yao et~al.(2019)Yao, Zhu, and Huang]{yao2019deep}
Yao, J., Zhu, X., and Huang, J.
\newblock Deep multi-instance learning for survival prediction from whole slide images.
\newblock In \emph{International Conference on Medical Image Computing and Computer-Assisted Intervention}, pp.\  496--504. Springer, 2019.

\bibitem[Yao et~al.(2020)Yao, Zhu, Jonnagaddala, Hawkins, and Huang]{yao2020whole}
Yao, J., Zhu, X., Jonnagaddala, J., Hawkins, N., and Huang, J.
\newblock Whole slide images based cancer survival prediction using attention guided deep multiple instance learning networks.
\newblock \emph{Medical Image Analysis}, 65:\penalty0 101789, 2020.

\bibitem[Yu et~al.(2022)Yu, Yap, Cheng, Ngo, Vaneckova, Karikios, Canfell, and Weber]{yu2022evaluating}
Yu, X.~Q., Yap, M.~L., Cheng, E.~S., Ngo, P.~J., Vaneckova, P., Karikios, D., Canfell, K., and Weber, M.~F.
\newblock Evaluating prognostic factors for sex differences in lung cancer survival: findings from a large australian cohort.
\newblock \emph{Journal of Thoracic Oncology}, 17\penalty0 (5):\penalty0 688--699, 2022.

\bibitem[Zadeh \& Schmid(2020)Zadeh and Schmid]{zadeh2020bias}
Zadeh, S.~G. and Schmid, M.
\newblock Bias in cross-entropy-based training of deep survival networks.
\newblock \emph{IEEE transactions on pattern analysis and machine intelligence}, 43\penalty0 (9):\penalty0 3126--3137, 2020.

\bibitem[Zadeh \& Schmid(2021)Zadeh and Schmid]{zadeh2021bias}
Zadeh, S.~G. and Schmid, M.
\newblock Bias in cross-entropy-based training of deep survival networks.
\newblock \emph{IEEE Transactions on Pattern Analysis and Machine Intelligence}, 43\penalty0 (9):\penalty0 3126--3137, 2021.
\newblock \doi{10.1109/TPAMI.2020.2979450}.

\bibitem[Zhang et~al.(2024)Zhang, Xu, Chen, Xie, and Chen]{zhang2024prototypical}
Zhang, Y., Xu, Y., Chen, J., Xie, F., and Chen, H.
\newblock Prototypical information bottlenecking and disentangling for multimodal cancer survival prediction.
\newblock In \emph{The Twelfth International Conference on Learning Representations}, 2024.

\bibitem[Zhou \& Chen(2023)Zhou and Chen]{zhou2023cross}
Zhou, F. and Chen, H.
\newblock Cross-modal translation and alignment for survival analysis.
\newblock In \emph{Proceedings of the IEEE/CVF International Conference on Computer Vision}, pp.\  21485--21494, 2023.

\end{thebibliography}
\bibliographystyle{icml2024}

\newpage
\appendix
\onecolumn

\section{Prototype-based histopathology baselines}\label{sec:histo_agg}
In this section, we present the three histology prototype aggregation approaches that can be used by $\ours$, with particular emphasis on the Gaussian mixture model (GMM). The following prototype-based aggregation schemes can be embedded as a feed-forward module in our models.

\subsection{Hard clustering (HC)} For each $\z_{i,\histo}$, we identify the closest prototype $\proto_{c,\histo}$ evaluated with the $\mathcal{L}_2$ distance, \textit{i.e.,} $c_i=\arg\max_c\lVert \z_{i,\histo} - \proto_{c,\histo}\rVert_2$ to determine the cluster assignment. The post-aggregation embedding $\z_{c,\histo}^{\text{agg.}}$ is an average of all embeddings assigned to $c$,
\begin{equation}
    \z_{c,\histo}^{\text{agg.}}=\sum_{i=1}^{N_\histo} \mathbf{1}_{c_i=c}\cdot\z_{i,\histo}/\sum_{i=1}^{N_\histo} \mathbf{1}_{c_i=c}.
\end{equation}
where $\mathbf{1}$ is the indicator function.

\subsection{Optimal transport (OT)} We can formulate aggregation as that of transporting from the empirical distribution of $\hat{p}(\z_{\histo})=1/N_\histo\cdot\sum_{i=1}^{N_\histo}\delta(\z_{i,h})$ to $\hat{p}(\proto_{\histo})=1/C_{\histo}\cdot\sum_{i=1}^{C_\histo}\delta({\proto_{c,\histo}})$. The transport plan $\mathbf{T}\in\mathbb{R}_{+}^{N_\histo\times C_{\histo}}$ is given as the solution to the following entropic-regularized optimal transport problem~\cite{cuturi2013sinkhorn, Kolouri2017Optimal},
\begin{equation}
\begin{split}
    &\min_{\mathbf{T}}\sum_{i,c}\lVert\z_{i,\histo} - \proto_{c,\histo}\rVert_2 \cdot \mathbf{T}_{i,c} +\epsilon\cdot \mathbf{T}_{i,c}\log \mathbf{T}_{i,c},\quad\text{such that } \sum_{i=1}^{N_\histo}\mathbf{T}_{i,c}=1/C_{\histo} \,\,\text{and}\,\, \sum_{i=1}^{C_{\histo}}\mathbf{T}_{i,c}=1/N_\histo,\\
\end{split}
\end{equation}
where $\varepsilon$ is the regularization parameter. Based on the optimal transport plan $\widehat{\mathbf{T}}$ obtained by the widely-used Sinkhorn algorithm~\cite{cuturi2013sinkhorn}, the post-aggregation embedding is given as $\z_{c,\histo}^{ \text{agg.}}=\sum_{i=1}^{N_\histo}\widehat{\mathbf{T}}_{i,c}\cdot \z_{i,\histo}$.

\subsection{Gaussian Mixture Models}
With the Gaussian mixture model (GMM) as the generative model for each token embedding, we provide a detailed derivation for estimation of 1) the posterior probability for the prototype assignment $q(c|\z_{i,\histo};\theta)$ and 2) the GMM parameters $\theta=\{\pi_c, \boldsymbol{\mu}_c, \Sigma_c\}$. Given the GMM specification,
\begin{equation}\label{eq:gmm_supp}
\begin{split}
   p(\z_{i,\histo} ; \theta) &= \sum_{c=1}^{C_\histo} p(c_i=c; \theta)\cdot p(\z_{i,\histo}| c_i=c;\theta)\\
   &= \sum_{c=1}^{C_\histo}\pi_c \cdot \mathcal{N}(\z_{i,\histo}; \boldsymbol{\mu}_c, \Sigma_c),\,\, s.t.\sum_{c=1}^{C_\histo}\pi_c=1,\\ 
\end{split}
\end{equation}
the goal is to estimate $\theta$ that maximizes the log-likelihood $\max_{\theta} \sum_{i=1}^{N_\histo}\log p(\z_{i,\histo};\theta)=\max_{\theta}\sum_{n=1}^{N_\histo}\log p(\mathbf{z}_{i,\histo};\theta)$.
We now present a detailed walkthrough of the expectation-maximization (EM) algorithm~\cite{dempster1977maximum, kim2022differentiable, song2024morphological} and how these ultimately lead to $\z_{c,\histo}^{\text{agg.}}$.

Using Jensen's inequality, we can lower-bound the log-likelihood as follows,
\begin{equation}
    \begin{split}
        \sum_{i=1}^{N_\histo}\log p(\mathbf{z}_{i,\histo};\theta)&=\sum_{i=1}^{N_\histo}\log \sum_{c=1}^{C_\histo} p(\z_{i,\histo}, c_i=c;\theta)\\
        &=\sum_{i=1}^{N_\histo}\log \sum_{c=1}^{C_\histo} q(c_i=c|\z_{i,\histo};\theta_{\text{old}})\cdot\frac{p(\z_{i,\histo}, c_i=c;\theta)}{q(c_i=c|\z_{i,\histo};\theta_{\text{old}})}\\
        &\geq \sum_{i=1}^{N_\histo}\sum_{c=1}^{C_\histo} q(c_i=c|\z_{i,\histo};\theta_{\text{old}}) \log \frac{p(\z_{i,\histo}, c_i=c;\theta)}{q(c_i=c|\z_{i,\histo};\theta_{\text{old}})}\\
        &= \sum_{i=1}^{N_\histo}\underbrace{E_{q(c_i=c|\z_{i,\histo};\theta_{\text{old}})}\left[\log p(\z_{i,\histo}, c_i=c;\theta)\right]}_{Q(\theta;\theta_{\text{old}})}-\sum_{i=1}^{N_\histo}\underbrace{E_{q(c_i=c|\z_{i,\histo};\theta_{\text{old}})}\left[q(c_i=c|\z_{i,\histo};\theta_{\text{old}})\right]}_{-H(C;\theta_{\text{old}})}.\\
    \end{split}
\end{equation}

Instead of maximizing the log-likelihood directly, we can now maximize a surrogate function, which is the lower bound given by Jensen's inequality. It can be shown that increasing this lower bound with respect to $\theta$ leads to monotonically increasing the actual log-likelihood~\cite{dempster1977maximum}. The optimization procedure involves iterative alternating steps of the E-step and the M-step and is thus referred to as the Expectation-Maximization (EM) algorithm. 

The surrogate function consists of two terms, $Q(\theta;\theta_{\text{old}})$ and $H(C;\theta_{\text{old}})$, which are expectations with respect to the posterior probability of prototype assignment, $q(c_i=c|\z_{i,\histo};\theta_{\text{old}})$. In the E-step, we can use Bayes' rule to compute the posterior probability and, consequently the expectations,
\begin{equation}
    \begin{split}
        q(c_i=c|\z_{i,\histo};\theta_{\text{old}}) &=\frac{q(\z_{i,\histo}|c_i=c;\theta_{\text{old}})\cdot q(c_i=c;\theta_{\text{old}})}{q(\z_{i,\histo};\theta_{\text{old}})}\\
        &=\frac{q(\z_{i,\histo}|c_i=c;\theta_{\text{old}})\cdot q(c_i=c;\theta_{\text{old}})}{\sum_{c=1}^{C_\histo} q(\z_{i,\histo}|c_i=c;\theta_{\text{old}})\cdot q(c_i=c;\theta_{\text{old}})}\\
        &=\frac{\pi_c \cdot \mathcal{N}(\z_{i,\histo}; \boldsymbol{\mu}_c, \Sigma_c)}{\sum_{c=1}^{C_\histo} \pi_c \cdot \mathcal{N}(\z_{i,\histo}; \boldsymbol{\mu}_c, \Sigma_c)}.\\
    \end{split}
\end{equation}

In the M-step, we find $\theta_{\text{new}}$ that maximizes the surrogate function based on the posterior probability computed from the E-step. Since the term $H(C;\theta_{\text{old}})$ is not a function of $\theta$ and therefore a constant (it is a function of $\theta_{\text{old}}$), we only need to optimize the term $Q(\theta; \theta_{\text{old}})$ by taking the derivative with respect to $\theta$,
\begin{equation}
    \begin{split}
        \sum_{i=1}^{N_\histo}\frac{\partial Q(\theta;\theta_{\text{old}})}{\partial \pi_c}=0 &\Rightarrow \pi_c^{\text{new}} = \frac{\sum_{i=1}^{N_\histo} q(c_i=c|\z_{i,\histo};\theta_{\text{old}})}{N_\histo}\\
        \sum_{i=1}^{N_\histo}\frac{\partial Q(\theta;\theta_{\text{old}})}{\partial \boldsymbol{\mu}_c}=0&\Rightarrow \boldsymbol{\mu}_c^{\text{new}}=\frac{\sum_{i=1}^{N_\histo} q(c_i=c|\z_{i,\histo};\theta_{\text{old}})\cdot\z_{i,\histo}}{\sum_{i=1}^{N_\histo} q(c_i=c|\z_{i,\histo};\theta_{\text{old}})} \\
        \sum_{i=1}^{N_\histo}\frac{\partial Q(\theta;\theta_{\text{old}})}{\partial \Sigma_c}=0&\Rightarrow \Sigma_c^{\text{new}}=\frac{\sum_{i=1}^{N_\histo} q(c_i=c|\z_{i,\histo};\theta_{\text{old}})\cdot (\z_{i,\histo} - \boldsymbol{\mu}_c^{\text{new}})^2}{\sum_{i=1}^{N_\histo} q(c_i=c|\z_{i,\histo};\theta_{\text{old}})}.\\
    \end{split}
\end{equation}

The E-step and M-step alternate until convergence is reached. In our setting, we usually found one round of EM iteration sufficient. As for the initial parameters, we set $\pi_c^{(0)}=1/C_\histo$, $\boldsymbol{\mu}_c^{(0)}=\mathbf{a}_{c,\histo}$, and $\Sigma_c^{(0)}=\mathbf{I}$, which serves as a morphology-aware initialization for the algorithm. 
The initialization for $\{\mathbf{a}_{c,\histo}\}_{c=1}^{C_\histo}$ is performed with K-means clustering on the training set of patches. This is constructed by aggregating token embeddings from all training slides in a disease cohort. 

Once $\widehat{\theta}$ is estimated, the post-aggregation embedding $\z_{c,\histo}^{\text{agg.}}\in\mathbb{R}^{d_\histo}$ with $d_\histo=1+2D$, can be represented as a concatenation $\z_{c,\histo}^{\text{agg.}}=[\widehat{\pi}_c, \widehat{\boldsymbol{\mu}}_c,\widehat{\Sigma}_c]$. Denoting $q_i=q(c_i=c|\z_{i,\histo};\theta_{\text{old}})$, we can express $\z_{c,\histo}^{\text{agg.}}$ as in Eq.~\ref{eq:histo_agg}, 
\begin{equation}
    \z_{c,\histo}^{\text{agg.}}=\sum_{i=1}^{N_\histo}\left[q_i/N_\histo,\,\,\, q_i\z_{i,\histo}/(\sum_{i=1}^{N_\histo}q_i),\,\,\, q_i\left(z_{i,\histo}-\sum_{i=1}^{N_\histo}q_i\z_{i,\histo}/(\sum_{i=1}^{N_\histo}q_i)\right)^2/(\sum_{i=1}^{N_\histo}q_i)\right],
\end{equation}
which can indeed expressed as a sum of the mapping function $g$ (albeit non-trivial to write out the full expression due to the iterative nature of EM) over $N_\histo$ elements.

\section{Proof for similarity between OT-based cross-alignment and Transformer-based cross-attention}\label{sec:proof}
\begin{lemma}
    Let $\Z_\gene\in\mathbb{R}^{C_\gene\times d}$ and $\Z_\histo\in\mathbb{R}^{C_\histo\times d}$ be the matrix representation of the token sets $\{\z_{i,\gene} \}_{i=1}^{C_\gene}$ and $\{\z_{k,\histo} \}_{k=1}^{C_\histo}$. Let $\Z_{\gene}\W_Q^{\text{T}}\in\mathbb{R}^{C_\gene\times d}$ and $\Z_{\histo}\W^{\text{T}}\in\mathbb{R}^{C_\histo\times d}$ be the linear projections of both sets. Let $\widehat{\mathbf{T}}\in\mathbb{R}^{C_\gene\times C_\histo}_{+}$ be the optimal transport plan, i.e., the solution to the entropic-regularized, unbalanced optimal transport problem between the two projected sets. Then, $\widehat{\mathbf{T}}$ is equivalent to the Transformer cross-attention matrix, $\sigma(\Z_\gene\W_Q^{\text{T}}\W\Z_\histo^{\text{T}}/\sqrt{d})$, up to a multiplicative factor where $\sigma(\cdot)$ denotes row-wise softmax, $\{\W_Q\z_{i,\gene}\}_{i=1}^{C_\gene}$ are queries, and $\{\W\z_{k,\histo}\}_{k=1}^{C_\histo}$ are keys.
\end{lemma}

\begin{proof}
This proof is an extension and adaption of a lemma from \cite{kim2022differentiable} to our application. 
We use the negative dot-product similarity as the cost between two sets of linearly-projected tokens $\{\W_Q\z_{i,\gene}\}_{i=1}^{C_\gene}$ and $\{\W\z_{k,\histo}\}_{k=1}^{C_\histo}$ as $\mathbf{D}_{i,k} = -\z_{i,\gene}^{\text{T}}\W_Q^{\text{T}}\W\z_{k,\histo}$. We can formulate the entropic-regularized optimal transport problem for optimizing the transport plan $\mathbf{T}\in\mathbb{R}_{+}^{C_\gene\times C_\histo}$,
\begin{equation}\label{eq:OT_app}
    \min_{\mathbf{T}}\sum_{i,k}\mathbf{D}_{i,k}\cdot \mathbf{T}_{i,k} + \varepsilon \mathbf{T}_{i,k}\log \mathbf{T}_{i,k},\quad\text{s.t.} \sum_{k=1}^{C_\histo} \mathbf{T}_{i,k} = \frac{1}{C_\gene},\forall i,
\end{equation}
without the constraint $\sum_{i=1}^{C_\gene} \mathbf{T}_{i,k}=1/C_\histo$. Note that this can be considered as an unbalanced OT problem~\cite{benamou2003numerical, chizat2018scaling}, as Eq.~\ref{eq:OT_app} can be written as
\begin{equation}
    \min_{\mathbf{T}}\sum_{i,k}\left(\mathbf{D}_{i,k}\cdot \mathbf{T}_{i,k} + \varepsilon \mathbf{T}_{i,k}\log \mathbf{T}_{i,k}\right) + \lambda_1\cdot \operatorname{Div.}( \mathbf{T}^{\text{T}}\cdot\mathbf{1}_{C_\gene}, 1/C_\histo\cdot \mathbf{1}_{C_\histo}) + \lambda_2\cdot \operatorname{Div.}( \mathbf{T}\cdot\mathbf{1}_{C_\histo}, 1/C_\gene\cdot \mathbf{1}_{C_\gene}),
\end{equation}
with $\lambda_1\rightarrow 0$ and $\lambda_2 \rightarrow \infty$, where $\operatorname{Div.}$ is some divergence measure and $\mathbf{1}_{C_\gene}$ is a $C_\gene$-length vector of ones. We now take Eq.~\ref{eq:OT_app} and solve it by using Lagrange multiplier,
\begin{equation}
    \mathcal{L} = \sum_{i,k}\left(\mathbf{D}_{i,k}\cdot \mathbf{T}_{i,k} + \varepsilon \mathbf{T}_{i,k}\log \mathbf{T}_{i,k}\right) +\sum_{i=1}^{C_\gene}\beta_i\left(\sum_{k=1}^{C_\histo} \mathbf{T}_{i,k} -\frac{1}{C_\gene}\right).
\end{equation}
We proceed by taking the derivative of $\mathcal{L}$ with respect to $\mathbf{T}_{i,k}$ and setting it to 0, 
\begin{equation}
    \frac{\partial \mathcal{L}}{\partial \mathbf{T}_{i,k}}=\mathbf{D}_{i,k} + \varepsilon(\log \mathbf{T}_{i,k} + 1)+\beta_i=0 \Rightarrow \mathbf{T}_{i,k}=\exp\left( -\mathbf{D}_{i,k}/\varepsilon +\gamma_i\right),
\end{equation}
where $\gamma_i=-(\beta_i/\varepsilon+1)$ is some constant. To solve for $\gamma_i$, we can use the constraint $\sum_{k=1}^{C_\histo} \mathbf{T}_{i,k}=\frac{1}{C_\gene}$,
\begin{equation}
    \exp(\gamma_i)\sum_{k=1}^{C_\histo}\exp(-\mathbf{D}_{i,k}/\varepsilon)=\frac{1}{C_{\gene}}\Rightarrow \exp(\gamma_i)=\frac{1}{C_\gene\cdot \sum_{k=1}^{C_\histo}\exp(-\mathbf{D}_{i,k}/\varepsilon)},
\end{equation}
and obtain $\widehat{\mathbf{T}}_{i,k}$ (by also setting $\varepsilon=\sqrt{d}$),
\begin{equation}
    \widehat{\mathbf{T}}_{i,k}=\frac{\exp(-\mathbf{D}_{i,k}/\sqrt{d})}{C_\gene\cdot \sum_{k=1}^{C_\histo}\exp(-\mathbf{D}_{i,k}/\sqrt{d})}=\frac{\exp(\z_{i,\gene}^{\text{T}}\W_Q^{\text{T}}\W\z_{k,\histo}/\sqrt{d})}{C_\gene\cdot \sum_{k=1}^{C_\histo}\exp(\z_{i,\gene}^{\text{T}}\W_Q^{\text{T}}\W\z_{k,\histo}/\sqrt{d})},
\end{equation}
with the softmax term appearing as an entry for $\widehat{\mathbf{T}}$. This is the same as the Transformer-based cross-attention operation up to a multiplicative factor of $1/C_\gene$.
\end{proof}

\section{Survival loss functions}\label{section:loss_supp}

Survival analysis models the time to an event, where the event outcome is not always observed (\textit{i.e.,} censored). In cancer survival outcome prediction, a censored event refers to patient survival or last known follow-up time, whereas an uncensored event is a patient death. Let $T$ be a continuous random variable representing patient survival time, and the survival function $S(t) = P(T \geq t_0)$ be the probability of a patient surviving longer than time $t_0$. The goal of survival analysis is to estimate the hazard function, $\lambda(t)$, which denotes the probability of an event occurring instantaneously at time $t>t_0$~\cite{cox1972regression}. We now detail the Cox proportional Hazards and Negative log-likelihood survival losses.


\subsection{Cox proportional hazards loss}
Cox proportional hazards model parameterizes the hazard function as an exponential linear function $\lambda(t|x)=\lambda_0(t)\exp^{\theta x}$. $\lambda_0$ is the baseline hazard function describing how the risk of an event changes over time. $\theta$ are the model parameters describing how the hazards vary with the features of a patient, $\bar{\mathbf{x}}_{\text{patient}} \in \mathbb{R}^{2d}$. To express the likelihood of an event to be observed at time $t$ with model parameters $\theta$, the Cox partial log-likelihood can be used ~\cite{wong1986theory}:

\begin{equation}\label{eq:cox_model}
l(\theta, \bar{\mathbf{x}}_{\text{patient}}) = -\sum_{i\in U}\biggl(\bar{\mathbf{x}}_{\text{patient}, i}\theta - \log(\sum_{j \in R_{i}} \exp({\bar{\mathbf{x}}_{\text{patient}, j}\theta}))\biggr)
\end{equation}

\begin{equation}\label{eq:cox_loss}
\frac{\partial l(\theta, \bar{\mathbf{x}}_{\text{patient}})}{\partial \bar{\mathbf{x}}_{\text{patient}, i}}=\delta(i)\theta-\sum_{i, j \in C_j, U}\frac{\theta\exp(\bar{\mathbf{x}}_{\text{patient}, i}\theta)}{\sum_{k\in C_j}\exp(\bar{\mathbf{x}}_{\text{patient}, k}\theta)}
\end{equation}

where $U$ is the set of uncensored patients, $C$ is the set of censored patients, $R_i$ is the set of patients whose last time of follow-up or time of death is after $i$, and $\delta(i)$ signifies if event outcome is observed or if censored. 

\subsection{Negative log-likelihood loss}
The Negative log-likelihood (NLL) survival loss~\cite{zadeh2021bias} generalizes the NLL to censored data. The aim is to predict the survival of a patient from the learned patient level embedding $\bar{\mathbf{x}}_{\text{patient}} \in \mathbb{R}^{2d}$. In accordance with previous work~\cite{zadeh2021bias}, the patient's survival state is defined by: (1) censorship status $c$, where $c=0$ represents an observed patient death due to disease and $c=1$ corresponds to the patient's last known follow-up, and (2) a time-to-event $t_i$, which corresponds to the time between the patient's diagnosis and observed death if $c=0$, or the last follow-up if $c=1$. Instead of predicting the observed time-to-event $t_i$, we \textit{discretize} it by defining non-overlapping time intervals $(t_{j-1}, t_j), \;j \in [1, ..., n]$ based on the quartiles of survival time values, and denote as $y_j$. The setup simplifies to a classification problem with censorship information, where each patient is now defined by $(\bar{\mathbf{x}}_{\text{slide}}, y_j, c)$. Next, we build a classifier such that each output logit $\hat{y}_j$ corresponds to a time interval. Then, we define the discrete hazard function $f_{\text{hazard}}(y_j | \bar{\mathbf{x}}_{\text{patient}}) = S(\hat{y}_j)$ where $S$ is the sigmoid activation. Intuitively, $f_{\text{hazard}}(y_j | \bar{\mathbf{x}}_{\text{patient}})$ represents the probability that the patient dies during time interval $(t_{j-1}, t_j)$. Additionally, we define the discrete survival function $f_{\text{surv}}(y_j | \bar{\mathbf{x}}_{\text{patient}}) = \prod_{k=1}^j \big(1 - f_{\text{hazard}}(y_k | \bar{\mathbf{x}}_{\text{patient}})\big)$ that represents the probability that the patient survives up to time interval $(t_{j-1}, t_j)$. Now, the NLL survival loss can be formally defined as:
\begin{align}\label{eq:rank_loss}
    \mathcal{L}\Big(&\{\bar{\mathbf{x}}^{(i)}_{\text{patient}}, y^{(i)}_j, c^{(i)} \}_{i=1}^{N_{D}} \Big) =\\
    &\sum_{i=1}^{N_{D}}
    -c^{(i)} \log(f_{\text{surv}}(y_j^{(i)} | \bar{\mathbf{x}}^{(i)}_{\text{patient}})) \label{eq:nnl_1} \\ 
    &+ (1-c^{(i)}) \log(f_{\text{surv}}(y_j^{(i)} -1 | \bar{\mathbf{x}}^{(i)}_{\text{patient}})) \label{eq:nnl_2} \\ 
    &+ (1-c^{(i)}) \log(f_{\text{hazard}}(y_j^{(i)} | \bar{\mathbf{x}}^{(i)}_{\text{patient}})) \label{eq:nnl_3}
\end{align}
where $N_{D}$ is the number of samples in the dataset. Eq.~\ref{eq:nnl_1} enforces high survival probability for patients alive after the final follow-up, Eq.~\ref{eq:nnl_2} enforces high survival up to the time stamp where death was observed for patients that died, and Eq.~\ref{eq:nnl_3} ensures correct timestamp is predicted for patients with observed death. As NLL does not require a set of patients for training, unlike Cox loss, it has been the de-facto loss function for cancer survival prediction with histology data, with the large number of tokens rendering the formation of patient batch infeasible.

\subsection{Concordance Index}
The Concordance Index (C-Index)~\cite{harrell1982evaluating} is a popular metric to measure the performance of survival prediction model ~\cite{chen2022pan, jaume2024modeling} and measures the rank correlation between the predicted risk scores and observed time points $t$. In prognosis prediction, the C-Index can be conceptually understood as a metric that assesses the accuracy of a model in predicting a higher risk of adverse outcomes for patients with shorter survival times. Formally, C-Index is defined as the ratio of concordant pairs to total comparable pairs. Two patients $i$ and $j$ are comparable if the patient with the lower observed time experienced an event (\textit{i.e.,} if $t_i > t_j$ then $\delta_j=1$, where $\delta$ is a binary indicator of whether event is observed or if last follow up time is known). A comparable pair $(i,j)$ is considered concordant if the risk predicted by a survival model $\hat{f}_{\text{risk}}$ is larger for the patient with the smaller event time, \textit{i.e.,} $\hat{f}_{\text{risk}, j} > \hat{f}_{\text{risk}, i}$ given $t_j < t_i$. Otherwise, the pair is considered discordant ~\cite{sksurv}. While C-Index allows for easy comparisons between models, known limitations exist, such as it is overly optimistic for increasing censorship in datasets~\cite{uno2011c}. 

\section{Datasets}\label{sec:dataset}
\subsection{TCGA cohort}
We evaluate all baselines on 6 cancer cohorts from TCGA: Bladder Urothelial Carcinoma (BLCA), Breast Invasive Carcinoma (BRCA), Lung adenocarcinoma (LUAD), Stomach adenocarcinoma (STAD), Colon and Rectum adenocarcinoma (CRC), Kidney renal clear cell carcinoma (KIRC), and low-grade gliomas (LGG). \textbf{Table~\ref{tab:tcga}} contains representative statistics of the dataset. A WSI is tessellated into nonoverlapping patches (tokens) of $256 \times 256$ pixels at $20\times$ magnification ($0.5\mu m$/pixel).

\begin{table}[!ht]
\centering
\small
\caption{\textbf{TCGA cohort statistics} The number of patients, total WSIs, and the average number of patches (tokens) in a WSI. A single patient can have multiple WSIs.}
\begin{tabular}{l|ccc}
\toprule
& Num. of patients & Num. of slides & Avg. set size \\
\midrule
BLCA & 359 & 423 & 16,312 \\
BRCA & 868 & 928 & 11,565 \\
LUAD & 412 & 463 & 4,714 \\
STAD & 318 & 318 & 10,955\\
CRC & 296 & 300 & 9,127\\
KIRC & 340 & 346 & 12,802 \\
\bottomrule
\end{tabular}
\label{tab:tcga}
\end{table}

\subsection{RNA-seq expression data}
Bulk RNA-seq expression for all TCGA cohorts---accessed from UCSC Xena database \cite{goldman2020visualizing}---is measured by Illumina HiSeq 2000 RNA Sequencing platform and then $\log_2(x+1)$ transformed RSEM normalized \cite{li2011rsem}. The $C_\gene=50$ Hallmark gene sets from Molecular Signatures Database (MSigDB) \cite{subramanian2005gene,liberzon2015molecular} are used to select and organize genes into biological pathways. Hallmark gene sets represent well-defined biological states in cancer. After organizing genes into Hallmark gene sets, we had 4,241 unique genes across the 50 gene sets. The average length of the gene sets is 142, with the minimum and maximum of 31 and 199. 


\section{Baselines}\label{section:baselines}
\subsection{Unimodal baselines}
In this section, we explain the unimodal MIL baselines that we compare our proposed framework with. 

\begin{enumerate}
    \item \textbf{ABMIL} \cite{ilse2018attention}: Attention-based multiple instance learning (ABMIL) first assigns patch-level importance scores through a local attention mechanism, where the score for one patch only depends on the contents of that patch. The attention-weighted sum of patches is used as the slide-level representation. The independence assumption of ABMIL neglects correlations between different patches. 
    
    \item \textbf{TransMIL} \cite{shao2021transmil}: Since ABMIL is unable to learn patch-level correlations, Transformer-based multiple instance learning (TransMIL) has been proposed. TransMIL first squares the sequence of low dimensional representations, then applies a Pyramidal Positional Encoding module to encode spatial knowledge, and finally uses Nystrom attention \cite{xiong2021nystromformer} to approximate self-attention scores between patches. The CLS token is taken as the slide-level representation.
    
    \item \textbf{Low-rank MIL} \cite{xiang2022exploring}: While TransMIL tries to learn slide-level representations by encoding patch correlations, it does not leverage the redundancy in WSI, which \cite{xiang2022exploring} used to propose iterative low-rank attention (ILRA). Each ILRA block consists of two layers: one aims to project the sequence of patch representations to a low-rank space by cross-attending it with a latent matrix, and the second reconstructs the input. Max-pooling over the output of $k$ such layers yields a low-rank slide-level representation.  

    \item \textbf{AttnMISL} \cite{yao2020whole}: In contrast with ABMIL, TransMIL, and ILRA, which learn slide-level representations using patch representations, AttnMISL first clusters patches into morphological prototypes using K-means clustering. Next, each prototype is encoded using prototype-specific fully convolutional Siamese networks \cite{yao2019deep}. The slide-level representation is then created using local attention pooling over the prototypes.  
    
    \item \textbf{Information Bottleneck MIL} \cite{li2023task}: Information bottlenecks (IB) are used to compress a WSI by removing irrelevant instances. IB aims to find patch instances that minimize the mutual information between the distribution of patches and patch representations. By only keeping such instances, \cite{li2023task} argue that most informative patches are retained, which can then be aggregated into a compact representation of WSI. 

    \item \textbf{Unimodal $\ours$}: This is similar to $\textsc{Panther}$~\cite{song2024morphological} in that GMM is used to map each histology patch embeddings into a pre-defined set of morphological prototypes. However, whereas $\textsc{Panther}$ concatenates the post-aggregation embeddings to form the slide representation, the unimodal $\ours$ employs $f_{\histo}^{\text{pre}}$, $f_{c,\gene}^{\text{pre}}$, and $f_c^{\text{post}}$ that are learned along with the downstream tasks.

\end{enumerate}

\subsection{Multimodal baselines}
We compare our proposed method $\ours$ with several early-fusion multimodal survival baselines.

\begin{enumerate}
    \item \textbf{MCAT} \cite{chen2021multimodal}: Multimodal Co-Attention Transformer (MCAT) is an early fusion technique that learns a dense co-attention mapping between histology and omic tokens. This mapping is then used to calculate omic-guided histology features, which are concatenated with omics to predict patient survival. MCAT uses omic prototypes because it groups genes into 6 functional families.  
    
    \item \textbf{MOTCat} \cite{Xu_2023_ICCV}: Multimodal Optimal Transport-based Co-attention Transformer (MOTCat) uses Optimal Transport to learn an optimal plan between histology tokens and genes grouped into 6 functional groups, similar to MCAT. The estimated optimal transport plan is then used for selecting the most informative histology tokens.  

    \item \textbf{SurvPath} \cite{jaume2024modeling}: Unlike MCAT and MOTCat, which are limited to six gene families, SurvPath introduces a transcriptomics tokenizer to encode genes into biological pathways that represent known cellular functions. The pathway tokens are then fused with histology patches via a memory-efficient transformer, which learns interactions between pathways and those between pathways and histology, but does not learn histology-to-histology interactions. 

    \item \textbf{CMTA} \cite{zhou2023cross}: Cross-Modal Translation and Alignment (CMTA) framework uses two parallel Transformer encoder-decoder modules. Encoders are used to extract intra-modal representations for each modality. Decoders generate cross-modal representations. A cross-modal attention module between the two encoders facilitates learning the cross-modal relations.
\end{enumerate}

\section{Clinical baselines}\label{sec:clinical}
We assess how $\ours$ and other survival prediction frameworks perform against basic clinical information included in patient metadata. Based on age, sex, and grade, empirically shown as crucial prognostic factors~\cite{bonnier1995age, rakha2010breast, tas2013age, yu2022evaluating}, we perform univariate/multivariate linear Cox regression to obtain the baseline. We observe that $\ours$ outperforms the baseline overall, hinting at its clinical potential for patient prognosis.

\begin{table*}[!ht]
\centering
\small
\caption{\textbf{Survival prediction with clinical variables} The clinical variables for the TCGA cohort were downloaded from cBioPortal. All denotes the combination of age, sex, and grade.}
\begin{tabular}{l|c|c|c|c|c|c|c}
\toprule
\textbf{Dataset} & BRCA & BLCA & LUAD & STAD & CRC & KIRC & Avg. ($\uparrow$)\\
\midrule
Age & 0.496{\tiny $\pm0.086$} & 0.578{\tiny $\pm0.056$} & 0.533{\tiny $\pm0.063$} & 0.449{\tiny $\pm0.055$} & 0.357{\tiny $\pm0.161$} & 0.554{\tiny $\pm0.147$} & 0.495  \\
Sex & 0.490{\tiny $\pm0.011$} & 0.489{\tiny $\pm0.028$} & 0.480{\tiny $\pm0.049$} & 0.529{\tiny $\pm0.069$} & 0.542{\tiny $\pm0.070$} & 0.437{\tiny $\pm0.057$} & 0.495  \\
Grade & 0.597{\tiny $\pm0.078$} & 0.515{\tiny $\pm0.018$} & N/A & 0.552{\tiny $\pm0.055$} & N/A & 0.594{\tiny $\pm0.083$} & N/A  \\
All & 0.563{\tiny $\pm0.055$} & 0.570{\tiny $\pm0.033$} & 0.528{\tiny $\pm0.028$} & 0.592{\tiny $\pm0.044$} & 0.655{\tiny $\pm0.119$} & 0.602{\tiny $\pm0.066$} & 0.585  \\
\textbf{$\ours_{\text{Trans.}}$} & 0.738{\tiny $\pm0.069$} & 0.635{\tiny $\pm0.051$} & 0.642{\tiny $\pm0.037$} & 0.598{\tiny $\pm0.051$} & 0.630{\tiny $\pm0.125$} & 0.747{\tiny $\pm0.106$}  & 0.665\\
\bottomrule
\end{tabular}
\label{tab:clinical_ablation}
\end{table*}

\section{Histology ablations}\label{sec:histology_ablations}

We perform additional experiments in four cancer types, varying: (1) the number of histology prototypes ($C_h=8,16,32$) and (2) the pretrained encoder (ResNet50, CTransPath, and UNI). The results are shown in \textbf{Tables}~\ref{tab:numproto_ablation}, \ref{tab:encoder_ablation}.

We observe that performance with UNI features is relatively consistent across $C_h$, with $C_h=32$ being the weakest. The choice of $C_h=16$ was influenced by two factors: (1) This gave the best overall performance in multimodal evaluation. (2) $C_h=8$ sometimes fails to distinguish between two similar but subtly different morphological exemplars (by grouping them into a single cluster), whereas $C_h=32$ induces harder morphological interpretation due to an excessive number of exemplars. $C_h=16$ offered the best trade-off.

\begin{table*}[!ht]
\centering
\small
\caption{\textbf{Ablation on the number of histology prototypes} Unimodal MMP was trained on varying number of histology prototypes $C_{\histo}$ for select cancer types.}
\begin{tabular}{l|c|c|c|c|c}
\toprule
\textbf{Dataset} & BRCA & BLCA & LUAD & CRC & Avg. ($\uparrow$)\\
\midrule
$C_{\histo}=8$ &0.720$\pm$ 0.06 & 0.601$\pm$ 0.04 & 0.592$\pm$ 0.04 & 0.641$\pm$ 0.11& 0.639\\
$C_{\histo}=16$ & 0.669$\pm$ 0.12 & 0.593 $\pm$ 0.06 & 0.600$\pm$ 0.04 & 0.646$\pm$ 0.11 & 0.627\\
$C_{\histo}=32$ &0.680 $\pm$ 0.09 & 0.590$\pm$ 0.05 & 0.587$\pm$ 0.04 & 0.617$\pm$ 0.13 & 0.619\\
\bottomrule
\end{tabular}
\label{tab:numproto_ablation}
\end{table*}

\begin{table*}[!ht]
\centering
\small
\caption{\textbf{Ablation on the histology encoder} Unimodal MMP was trained on different histology encoders for select cancer types.}
\begin{tabular}{l|c|c|c|c|c}
\toprule
\textbf{Dataset} & BRCA & BLCA & LUAD & CRC & Avg. ($\uparrow$)\\
\midrule
ResNet50 & 0.574$\pm$ 0.11 & 0.511$\pm$ 0.05 & 0.600$\pm$ 0.06 & 0.534$\pm$ 0.18 & 0.555 \\
CTransPath & 0.653$\pm$ 0.10 & 0.566$\pm$ 0.05 & 0.578 $\pm$ 0.02 & 0.574 $\pm$ 0.14 & 0.593\\
UNI & 0.669 $\pm$ 0.12 & 0.593 $\pm$ 0.06 & 0.600 $\pm$ 0.04 & 0.646 $\pm$ 0.11 & 0.627\\
\bottomrule
\end{tabular}
\label{tab:encoder_ablation}
\end{table*}

\section{Survival loss ablation experiment}\label{sec:surv_ablation}
We assess how the train batch size affects the performance, using the Cox proportional hazards loss~\cite{cox1972regression} and NLL survival loss~\cite{zadeh2020bias} (\textbf{Table~\ref{tab:survival_loss}}). To this end, we use the $\ours$ full model. 
We observe that the C-Index increases with a larger batch size until it reaches the peak and starts to decline, regardless of the loss function (peak for Cox loss: 0.665 with batch size 64 and NLL loss: 0.644 with batch size 16). The increase can be attributed to stable training from having more patients in each batch to compare the predicted risks against~\cite{kvamme2019time}. The decrease is likely due to a smaller number of parameter updates within the same number of epochs. This suggests the benefits of batch-based training for survival prediction, which does not apply to non-prototype-based approaches as they rely on the NLL survival loss with a single patient batch due to large $N_\histo$. We also observe that employing the Cox loss gives a better overall performance, which illustrates another benefit of forming a batch of patients in $\ours$ with fewer tokens. We attribute the lower performance of NLL survival loss to the discretization of time into non-overlapping coarse bins, which might result in discarding valuable survival information.

\begin{table}[!ht]
\centering
\small
\caption{\textbf{Batch size ablation.} Average C-Index across 5 cross-validation folds with varying batch sizes of patients with Cox and NLL loss. A batch of a single patient cannot be used for Cox loss. }
\begin{tabular}{ll|c|c|c|c|c|c|c}
\toprule
&& BRCA & BLCA & LUAD & STAD & CRC & KIRC & Avg.$(\uparrow)$ \\
\midrule
\parbox[t]{0mm}{\multirow{5}{*}{\rotatebox[origin=c]{90}{{\textbf{Cox}}}}} &
$B=1$& N/A & N/A & N/A & N/A & N/A & N/A & N/A\\
&$B=16$ & 0.711 & \textbf{0.642} & 0.648 & 0.558 & \textbf{0.635} & 0.730 & 0.654\\
&$B=32$ & 0.729 & 0.636 & \textbf{0.648} & 0.584 & 0.627 & 0.735 & 0.660\\
&$B=64$ & \textbf{0.738} & 0.635 & 0.645 & \textbf{0.598} & 0.630 & \textbf{0.744} & \textbf{0.665}\\
&$B=128$ & 0.729 & 0.622 & 0.644 & 0.586 & 0.617 & 0.731 & 0.655\\
\midrule
\parbox[t]{0mm}{\multirow{5}{*}{\rotatebox[origin=c]{90}{{\textbf{NLL}}}}} &
$B=1$ & 0.664 & 0.602 & 0.616 & 0.508 & 0.627 & \textbf{0.712} & 0.621\\
&$B=16$ & \textbf{0.662} & \textbf{0.635} & \textbf{0.656} & 0.561 & \textbf{0.660} & 0.691 & \textbf{0.644}\\
&$B=32$ & 0.618 & 0.622 & 0.646 & \textbf{0.570} & 0.554 & 0.690 & 0.617\\
&$B=64$ & 0.590 & 0.616 & 0.635 & 0.556 & 0.574 & 0.678 & 0.608\\
&$B=128$ & 0.587 & 0.610 & 0.623 & 0.523 & 0.523 & 0.640 & 0.584\\
\bottomrule
\end{tabular}
\label{tab:survival_loss}
\end{table}

\begin{figure}[!ht]
   \centering
   \includegraphics[width=1\linewidth]{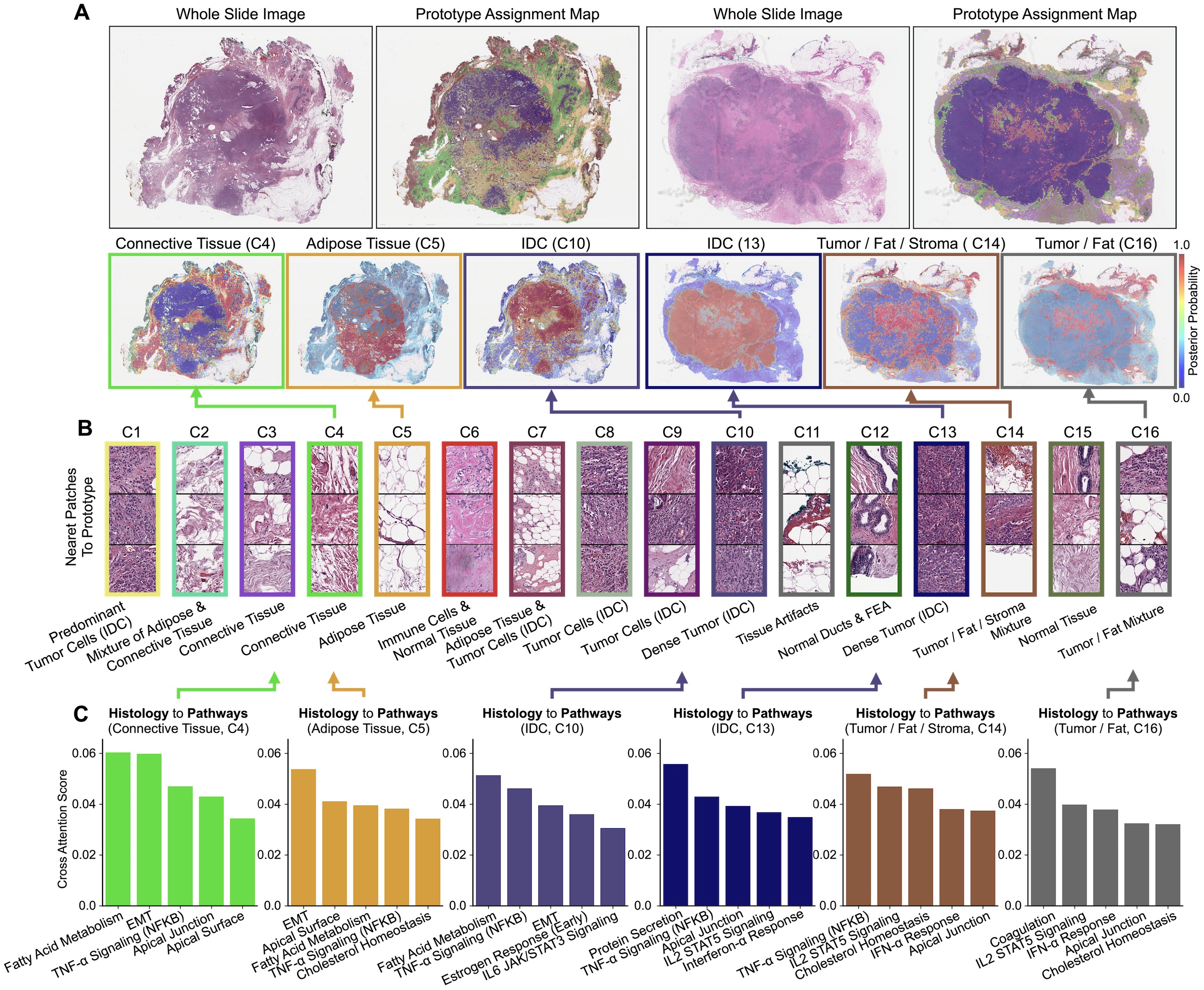}
   \vspace{-5mm}
   \caption{\textbf{Cross-modal interaction visualization}. (A) BRCA WSIs with their prototype assignment map (categorical assignment of each histology patch to their nearest prototype), and prototype heatmaps of the the top-3 prominent tissue patterns in the WSI. (B) Morphological annotations provided by a board-certified pathologist of the nearest histology patches for each prototype. (C) For each prototype visualized in (A), we can visualize its most highly-attended pathways (h. $\rightarrow$ g.), \textit{i.e.}, which pathways correspond to the queried prototype (pathway importance).}
   \label{fig:heatmap_supp_BRCA}
\end{figure}

\section{Additional interpretability results}\label{sec:heatmaps}

    \textbf{Unimodal:} Compared to \textbf{Fig.~\ref{fig:heatmap}}, \textbf{Fig.~\ref{fig:heatmap_supp_BRCA}A} visualizes multiple prototype heatmaps, illustrating how different prototypes reflect distinct morphological tissue patterns in the tumor microenvironment. Using their nearest histology patches, a board-certified pathologist qualitatively assessed and captioned each prototype with a general morphological description. We found that using $C_\histo=16$ may still have redundancy in unique prototypes, as multiple prototypes are found to describe IDC presence (C1, C8, C9, C10, C13). In general, however, each prototype was still found to be semantic in delineating general tumor tissue, normal connective tissue and stroma, adipose tissue, and tissue with immune cell presence, which is reflected in the high performance of unimodal MMP over other histology baselines found in \textbf{Table~\ref{tab:survival_main}} (with MMP using $C_\histo=8$ and $C_\histo=32$ having worse performance in \textbf{Table~\ref{tab:survival_ablation}}).

\textbf{Multimodal:} In \textbf{Fig.~\ref{fig:heatmap_supp_BRCA}C}, we further visualize cross-modal histology to pathway interactions (h. $\rightarrow$ g.) in BRCA, a unique capability in MMP compared to other works which have only visualized pathway to histology interactions (g. $\rightarrow$ h.) \cite{chen2021multimodal,xu2023multimodal,jaume2024modeling}. Across all (h. $\rightarrow$ g.) visualizations for the prototypes shown in \textbf{Fig.~\ref{fig:heatmap_supp_BRCA}C}, fatty acid metabolism and cholesterol homeostasis were conserved in having high cross-attention scores, which corroborates with biomedical literature on how cancer cells hijack these pathways for exogenous energy uptake from the tissue microenvironment (enabling tumorigenesis and cancer progression) \cite{nelson2014cholesterol,koundouros2020reprogramming}. Other conserved and highly-attended pathways include tumor necrosis factor (TNF)-$\alpha$ signaling and epithelial-mesenchymal transition (EMT), which are canonical markers related to tumor proliferation and invasion \cite{wu2010tnf,dongre2019new}. We note that SurvPath also found EMT to have high importance, however, we note a subtle difference in that EMT importance was derived from attribution-based interpretability with respect to predicted survival risk, and not via cross-attention that pinpoints a relationship with an exact morphological pattern. In MMP, we find that EMT not only attends to invasive tumor (C10), but is also the most highly-attended pathway to adipose tissue (C5), which corroborates with recent and accumulating evidence of adipose tissue being more than a causal observer in contributing to inflammation and tumor progression \cite{wang2012adipose,olea2018signaling,giudetti2019specific,ishay2019gain,olea2020new,loo2021fatty}.

\section{Limitations \& recommendations for future directions}\label{sec:app_limitations}

\textbf{Multimodal interpretability:} Due to potential redundancy of prototypes (corresponding to unique morphological patterns), queries for (h. $\rightarrow$ g.) are not unique, with many prototypes associated with tumor cell presence of IDC morphology and thus querying similar pathways. In \textbf{Fig.~\ref{fig:heatmap}}, we also note potential asymmetrical relationships in histology-pathway correspondences, in which cholesterol homeostasis highly attends to C13 (top 3 pathways out of 50) but C13 does not attend as highly to cholesterol homeostasis (top 6 prototypes out of 16). Again, this may be due to the redundancy of prototypes, with other IDC-related prototypes (C8) highly attending to cholesterol homeostasis instead. We note that though pathologist annotation found many clusters to correspond to similar morphological patterns for tumors, there may exist subtle differences in fine-grained features such as tumor grade, tumor invasiveness, tumor colocalization with stroma, adipose tissue, and immune cells which may have fine-grained interactions to pathways. Future directions include developing approaches that would narrow down the number of unique prototypes, which may improve both survival modeling and cross-modal interpretability.

\textbf{Study designs involving TCGA}: The TCGA is the largest publicly-available pan-cancer atlas with paired histology-omic samples, and has been an immeasurable resource for the CPath community in building computational tools for unimodal and multimodal cancer prognosis. Still, the TCGA has several limitations which we provide caution. First, in addition to issues such as site-specific H\&E intensity bias \cite{howard2021impact}, and demographic bias \cite{vaidya2024demographic}, pretrained encoders developed on the TCGA should also be avoided when evaluating multimodal cancer prognosis tasks due to potential issues in data contamination. Though UNI was not pretrained on TCGA \cite{chen2024towards}, using UNI (or any pretrained ROI encoder) as a part of non-parametric methods such as K-means clustering or GMMs may still lead to instances where all patches can be assigned to a single prototype, as demonstrated in $\textsc{PANTHER}$ \cite{song2024morphological}. Second, important consideration must be taken in utilizing the different survival endpoints available for each TCGA cohort. For instance, the median time-to-event and time-to-censor for disease-specific survival TCGA-BRCA is 26 and 25 months respectively, meaning that the follow-up time is too short to see breast cancer-specific deaths. Other works which have assessed the suitability of DSS as a survival endpoint in TCGA-BRCA were able to still show statistically significant differences between ER+ and ER- tumors, while also acknowledging its shortcomings \cite{liu2018integrated}.

\textbf{Unimodal versus multimodal survival analysis}: As emphasized in the \textbf{Introduction} and \textbf{Related Work} sections, multimodal survival analysis is a challenging clinical task that has seen significant interest in the biomedical,  computer vision, and machine learning communities. Though multimodal integration generally outperforms unimodal baselines, we note that the development of better unimodal baselines may (or may not) close the performance gap for certain cancer types, which is an area of further exploration. In PORPOISE \cite{chen2022pan} and MCAT \cite{chen2021multimodal}, multimodal integration (using ResNet50 features transferred from ImageNet for histology and gene families from MutSigDB for genomics) was found to improve in 9 out of 14 cancer types in the TCGA, with genomics generally outperforming histology in unimodal baselines. In SurvPath \cite{jaume2024modeling}, MOTCat \cite{xu2023multimodal} and PIBD \cite{zhang2024prototypical}, which improved unimodal baselines in MCAT using CTransPath features and hallmark gene family features, also found very similar trends with multimodal improvement. Interestingly, MCAT was shown to lag behind unimodal genomics in the analysis of SurvPath, which may be attributed to not only stronger gene features used, but also higher computational complexity with the increased number of omics tokens used for Transformer attention (thus necessitating computational efficiency). In MMP, which improves the unimodal histology baseline further using UNI features, we observe that the unimodal ablation of MMP (based on GMM, 0.611 overall C-Index) is able to catch up with unimodal genomics baselines (0.612 to 0.614 overall C-Index) and also with multimodal baselines like MCAT (0.610 overall C-Index) (\textbf{Table~\ref{tab:survival_main}}). We hypothesize that this is due to the simplicity of GMMs in representing WSIs as a fixed set of prototypes, which thus allows supervision using the Cox loss instead of the negative log-likelihood loss. As better unimodal baselines are developed, we envision new types of multimodal fusion techniques will also be needed that would emphasise simplicity and interpretability in developing easy-to-train survival methods in high-dimensional, low-sample size regimes for cancer prognostication.

\end{document}